%% file: main.tex
\documentclass[runningheads]{llncs}
\usepackage[T1]{fontenc}
\usepackage{multirow}
\usepackage[misc]{ifsym}

\usepackage{graphicx}   
\usepackage{url}        
\usepackage{amssymb}
\usepackage{amsmath}
\usepackage{dsfont}
\usepackage{upgreek}
\usepackage{bbm}			%
\usepackage{tikz}
\usetikzlibrary{arrows.meta,positioning,calc}
\usepackage{pgfplots}
\pgfplotsset{compat=1.18}
\usepackage{xspace}
\usetikzlibrary{arrows}
\usepackage{algorithm}
\usepackage{algorithmicx}
\usepackage{algpseudocode}
\usepackage{listings}
\newcommand{\argmin}{\operatornamewithlimits{argmin}}

\renewcommand*{\top}{{\mkern-1.5mu\mathsf{T}}}
\newcommand{\ind}       {\mathds{1}}%
\newcommand{\Ind}[1]    { \ind{\{#1\}} }

\usepackage[labelfont=bf]{caption}
\usepackage{subcaption}

\usepackage{xcolor}
\usepackage[hidelinks]{hyperref}
\hypersetup{
    colorlinks,
    linkcolor={red!80!black},
    citecolor={blue!50!black},
    urlcolor={blue!80!black}
}

\usepackage{booktabs,arydshln}
\makeatletter
\def\adl@drawiv#1#2#3{%
        \hskip.5\tabcolsep
        \xleaders#3{#2.5\@tempdimb #1{1}#2.5\@tempdimb}%
                #2\z@ plus1fil minus1fil\relax
        \hskip.5\tabcolsep}
\newcommand{\cmidruledashed}[1]{%
  \noalign{\vskip\aboverulesep
           \global\let\@dashdrawstore\adl@draw
           \global\let\adl@draw\adl@drawiv}
  \cdashline{#1}%
	\noalign{\vskip-\belowrulesep}
	\noalign{\vskip-\belowrulesep}
}
\makeatother
\usepackage{multirow}

\makeatletter
\newcommand{\pushright}[1]{\ifmeasuring@#1\else\omit\hfill$\displaystyle#1$\fi\ignorespaces}
\newcommand{\pushleft}[1]{\ifmeasuring@#1\else\omit$\displaystyle#1$\hfill\fi\ignorespaces}
\makeatother
\usepackage[normalem]{ulem} %

\newcommand{\VOID}[1] 				{}

\newcounter{marginNoteCounter}

\newcommand{\mySqBullet}		{\raisebox{0.25em}{{\scriptsize$_\blacksquare$}}}

\newcommand{\inlinetitle}[2]  {\vspace{4pt}\noindent\textbf{\emph{#1}{#2}}}

\newcommand{\Sec}[1]		{Sec.\,\ref{#1}}

\newcommand{\Fig}[1]		{Fig.\,\ref{#1}}\newcommand{\Eq}[1]			{Eq.\,\ref{#1}}\newcommand{\Tab}[1]		{Tab.\,\ref{#1}}\newcommand{\Alg}[1]		{Alg.\,\ref{#1}}%
\newcommand{\ie}   			{i.e.\@\xspace}
\newcommand{\eg}   			{e.g.\@\xspace}

\newcommand{\LIME}[1]	{{#1-LIME}\xspace}
\newcommand{\NDTLIME}	{{\LIME{NDT}}\xspace}

\usepackage{authblk} %

\title{Enhancing LIME using Neural Decision Trees}

\authorrunning{M.A. Bouyahia and A. Kalogeratos}

\author{Mohamed Aymen Bouyahia\inst{1,2} \and Argyris Kalogeratos\inst{1}}

\institute{ENS Paris-Saclay, Centre Borelli, CNRS, Gif-Sur-Yvette, France \and Le Crédit Lyonnais, Paris, France\\
\email{mohamed-aymen.bouyahia@lcl.fr}, \ \email{argyris.kalogeratos@ens-paris-saclay.fr}}

\begin{document}

\title{Enhancing LIME using Neural Decision Trees}

\maketitle

\begin{abstract}
Interpreting complex machine learning models is a critical challenge, especially for tabular data where model transparency is {para-mount}. Local Interpretable Model-Agnostic Explanations (LIME) has been a very popular framework for interpretable machine learning, also inspiring many extensions. While traditional surrogate models used in LIME variants (\eg linear regression and decision trees) offer a degree of stability, they can struggle to faithfully capture the complex non-linear decision boundaries that are inherent in many sophisticated black-box models. 
This work contributes toward bridging the gap between high predictive performance and interpretable decision-making. %
Specifically, we propose the \NDTLIME variant that integrates Neural Decision Trees (NDTs) as surrogate models. By leveraging the structured, hierarchical nature of NDTs, our approach aims at providing more accurate and meaningful local explanations. We evaluate its effectiveness on several benchmark tabular datasets, showing consistent improvements in explanation fidelity over traditional LIME surrogates.
\end{abstract}

\section{Introduction}
\label{sec:intro}
In recent years, machine learning models have achieved remarkable success across a variety of domains, including healthcare, finance, and marketing. Many state-of-the-art models, though, especially complex ones like deep neural networks and ensemble methods, are often considered “black boxes” due to their lack of transparency. This opacity poses a significant barrier to their adoption in high-stakes settings where understanding the rationale behind a prediction is crucial for trust, accountability, and regulatory compliance.
Critically, in domains where decisions must be both accurate and interpretable, such as healthcare and finance, tabular data is the most common form of data. Despite the success of deep learning in several other contexts, building foundation models and beating tree-based models on tabular data \cite{Deep_learning_tabular,foundation_model_tabular} remain big challenges. This is because handling tabular data is a setting in which until recently tree-based models have been consistently better due to their inherent inductive biases, notably since they handle uninformative features gracefully \cite{dts-tabular-2022}. This inherent advantage of tree-based models on tabular data makes them the preferred choice in these applications. 

Explainable artificial intelligence (XAI) techniques have emerged to provide insights into model behavior. Among these, Local Interpretable Model-Agnostic Explanations (LIME) \cite{lime} has gained popularity for generating interpretable local explanations for individual predictions. LIME uses traditional and rather simple surrogate models, like linear regression or decision trees, to capture and explain locally the decision boundaries of a given more complex model. %
Subsequent studies have shown that LIME can exhibit low local fidelity and sensitivity to perturbations \cite{GLime_paper} \cite{sLIME}, and it can also struggle to faithfully capture complex, non-linear decision boundaries with traditional surrogates \cite{LIME_limitations}. In simple terms, this highlights a trade-off between capacity and explainability of the surrogate models used. Those limitations have spurred the development of many variants.

As detailed in Table \ref{tab:lime-variants}, these include methods like Anchors \cite{anchors} for rule-based explanations, GLIME \cite{GLime_paper} for improved stability and fidelity via unbiased sampling, and LIMEtree \cite{LIMEtree_paper} utilizing tree-based ensembles for consistency. Despite these advancements, many variants still rely on simple surrogate models that may fall short in accurately representing intricate local decision manifolds.

In this paper, we propose the novel \NDTLIME variant of the LIME framework, that improves interpretability and robustness by integrating Neural Decision Trees (NDTs) \cite{Neural_Decision_Tree,to_tree_or_not_to_tree,li2022surveyneuraltrees} as local surrogate models. NDTs combine the expressive power of neural networks with the hierarchical, rule-based structure of decision trees, enabling more faithful and stable explanations. 
By initializing NDTs from traditional decision trees, our \NDTLIME approach directly leverages this established strength of tree-based models. Furthermore, by employing the differentiable approximation of splits inherent in NDTs, our method can generate smoother local explanations, providing a more nuanced view of decision boundaries compared to the rigid, axis-aligned splits of traditional decision trees. 
Our approach leverages this synergy to improve both local and global interpretability for tabular datasets. %
By employing NDTs, we aim at providing more accurate and meaningful local explanations, directly addressing the fidelity limitations often encountered with simpler LIME surrogates.

We validate \NDTLIME through experiments on benchmark tabular datasets, demonstrating superior explanation fidelity compared to traditional LIME surrogates. \Tab{tab:lime-variants} provides a shortlist of existing methods (for a survey, see the survey in \cite{which_lime}) and gives the positioning of the proposed \LIME{NDT} therein. Additionally, we provide tools for discretization, visualization, and decision tree analysis to facilitate deeper insights into model behavior. This work contributes a practical and effective method to bridge the gap between black-box predictive models and transparent, trustworthy AI.

\begin{table*}[t]
\centering
\resizebox{0.6\textwidth}{!}{%
\centerline{
\begin{tabular}{lllll}
\toprule
\textbf{LIME variant} & \textbf{Surrogate model} & \textbf{Improvement features} & \textbf{Limitations} & \textbf{Repository} \\
\midrule
LIME \cite{lime} & Linear & Standard LIME: perturbation-based local explanations & Instability, low fidelity for complex models & \href{https://github.com/marcotcr/lime}{GitHub} \\
Anchors \cite{anchors} & Rule-based & High-precision rules for local explanations & Explaining complex output spaces can be challenging & \href{https://github.com/marcotcr/anchor}{GitHub} \\
LORE \cite{LORE} & Decision Tree & Generates local logical rules and counterfactuals to explain decisions. & Requires a genetic algorithm for neighborhood generation & \href{https://github.com/riccotti/LORE}{GitHub} \\
GLIME \cite{GLime_paper} & Linear & Improved stability and local fidelity via unbiased sampling & Still limited by simple surrogate & \href{https://github.com/thutzr/GLIME-General-Stable-and-Local-LIME-Explanation}{GitHub} \\
DLIME \cite{DLIME} & Linear & Deterministic sampling using hierarchical clustering & May miss subtle patterns of model & \href{https://github.com/rehmanzafar/dlime_experiments}{GitHub} \\
MPS-LIME \cite{LIME_MPS} & Linear & Modified perturbation sampling to improve fidelity & Complex to implement, limited generalization &  -- \\
KL-LIME \cite{LIME_KL} & Linear & Uses Kullback-Leibler projection for Bayesian models & Limited to Bayesian models & -- \\
MeLIME \cite{MeLIME_paper} & Linear & Data-aware neighborhood sampling, robust convergence, counterfactuals & Requires domain knowledge for setup & \href{https://github.com/tiagobotari/melime?tab=readme-ov-file\#melime-features}{GitHub} \\
sLIME \cite{sLIME} & Linear & Stabilizes explanations by adaptively determining perturbations & Effectiveness can vary based on black-box model type & \href{https://github.com/ZhengzeZhou/slime}{GitHub} \\
ILIME \cite{ILIME_paper} & Linear + Influence & Combines proximity and influence functions for more faithful explanations & Requires access to model gradients for influence functions & -- \\
BayLIME \cite{BayLIME_paper} & Bayesian Linear Regr. & Improves consistency, robustness, and fidelity by leveraging priors & Risk of introducing bias with incorrect priors & \href{https://github.com/x-y-zhao/BayLime}{GitHub}\\
LIMEtree \cite{LIMEtree_paper} & Tree-based Ensemble & Greater consistency towards perturbations in local explanations & Less interpretable due to ensemble aggregation & -- \\
OptiLIME \cite{optimLIME} & Linear & Framework to optimize kernel width for adherence-stability trade-off & Computationally intensive due to Bayesian Optimization & \href{https://github.com/giorgiovisani/lime_stability/tree/master}{GitHub} \\
GMM-LIME \cite{GMM_LIME_paper} & Gaussian Mixture & Improved stability and potentially better local approximation & Increased model complexity and parameter tuning & -- \\
\midrule
\textbf{\NDTLIME (ours)} & Neural Decision Tree & Uses NDTs as surrogate models for higher fidelity & More computationally intensive & \href{https://github.com/aymen20002005/lime_ndt}{GitHub} \\
\bottomrule
\end{tabular}%
}%
}%
\vspace{0.5em}
\caption{A comprehensive list of LIME variants in the literature.}
\label{tab:lime-variants}
\end{table*}

\section{Background}
\label{sec:background}

\subsection{Local Interpretable Model-Agnostic Explanations}
LIME, standing for Local Interpretable Model-Agnostic Explanations \cite{lime}, is a widely adopted post-hoc method that explains individual predictions by approximating the behavior of a complex model ($f$) with a simpler and interpretable surrogate model ($g$), \eg linear regression or decision trees.
The key steps LIME follows to assess the decision $f(x)$, that is the output of the black-box model for a data instance of interest $x$, are:
\begin{enumerate}
    \item \emph{Perturbation}: Generate a synthetic dataset with $R$ instances $x'$ around the original instance $x$, by randomly perturbing feature values.
    \item \emph{Weighting}: Assign a weight $w$ to each perturbed $x'$ instance based on its proximity to $x$. %
    \item \emph{Surrogate model}: Train an interpretable model $g$ to imitate the output of the black-box on the weighted perturbed instances, $\{w, x', f(x')\}_{x'}$, and use it to explain the local decision boundary of $f$ around $x$.
\end{enumerate}\vspace{0.2cm}
Formally, LIME's explanation for a decision $f(x)$ writes: 
\begin{equation}\label{eq:lime-objective}
expl(x) = \argmin_{g\in G} L(f,g,\pi_x) + \Omega(g),
\end{equation}
where \(\pi_{x}\) is a proximity kernel that emphasizes samples closer to the instance of interest, $L(f,g,\pi_x)$ is a loss function measuring how well g approximates f in the locality defined by $\pi_x$, and $\Omega(g)$ is a complexity penalty that encourages simpler explanations.

While the LIME framework is generic, model-agnostic, and flexible, its practical instantiation is challenging due to a series of critical choices that need to be made, which have motivated many research works:
\begin{itemize}
		\item \emph{Surrogate model selection}: This choice affects 
		the decision boundary approximation. Simple models, \eg linear, may fail to capture non-linear or hierarchical decision boundaries inferred by complex black-box models, leading to explanations of low-fidelity. %
    \item \emph{Instability}: Explanations can vary significantly with minor changes in perturbations or sampling.%
		\item \emph{Tuning difficulty}: Sensitive parameters need to be tuned, such as to define the neighborhood kernel width, a suitable distance measure, and perturbation strategy for the locality $\pi_x$. 
\end{itemize}
Inadequate choices for the above can lead to unfaithful or unstable explanations. These challenges are extensively reviewed in the literature, for instance in the survey in \cite{which_lime}.

\inlinetitle{Surrogate models}{.}~%
Having the generic LIME framework, the choice of surrogate model is the main point on which the community has focused (see \Tab{tab:lime-variants}). Although the main principle is to choose a surrogate model that is properly simple and interpretable, in fact the model at the same time has to have sufficient capacity to capture the local decision boundaries of the black-box model.
Decision Trees (DTs) are a typical choice for surrogate models in LIME, for both classification and regression tasks, as they are inherently interpretable due to their \emph{rule-based structure} that mimics human reasoning \cite{LIMEtree_paper}.
However, traditional DTs are trained via greedy splitting, which can lead to suboptimal or unstable trees.
Moreover, their rigid splits may not adequately represent the decision boundaries of modern ML models, particularly in high-dimensional spaces \cite{dts-tabular-2022}. A number of LIME variants are briefly described in Appendix \ref{app:implementation-competitors}.

\subsection{Quality measures for surrogate models}
\label{sec:quality-measures}
 
We employ three quality metrics to evaluate surrogate models with respect to a black-box model, namely, \emph{fidelity}, \emph{stability}, and \emph{regularity}. These metrics are complementary %
as they jointly capture faithfulness, robustness to randomness, and smoothness of explanations. Each of the metrics computes a score for each of the $N=|D|$ datapoints of a given dataset $D$, and then the overall score for the dataset is computed by averaging the instance-wise scores:
\begin{equation}\label{eq:qmeasure-averaging}
\mathrm{QualMetric}(D) = \frac{1}{N}\sum_{i=1}^{N}\mathrm{QualMetric}(x_{i}).
\end{equation}
In the following, we denote by $E(x) \in \mathbb{R}^d$ the \emph{explanation vector} produced for the data instance $x\in D$, which is the feature importance vector, \ie the coefficients or weights derived from the trained local surrogate model.

\inlinetitle{Fidelity}{.}~The primary goal of a local surrogate model is to be faithful to the approximated black-box model. Local fidelity measures how well the explanation model mimics the predictions of the black-box model on the neighborhood data generated around the instance of interest. High fidelity indicates that the explanation is a trustworthy representation of the model's local behavior.

To quantify fidelity, we first generate %
$R$ perturbed samples around a chosen test instance. We then obtain predictions for this data from both the complex black-box model and the trained surrogate model. We measure fidelity using the Coefficient of Determination ($\mathrm{R}^2$) score \cite{Montgomery2012}, which calculates the proportion of the variance in the black-box model's predictions that is predictable from the surrogate model's predictions:
\begin{equation}
\mathrm{Fidelity}(x_{i}) := \mathrm{R}^2 = 1 - \frac{\sum_{j=1}^{n} [f(x'_{i,j}) - g(x'_{i,j})]^2}{\sum_{j=1}^{n} [f(x'_{i,j}) - \overline{Y}]^2},
\label{eq:fidelity_score_formula}
\end{equation}
 where $\bar{Y} = \frac{1}{n} \sum_{j=1}^{n} f(x'_{i,j})$. This index takes values in $(-\infty, 1]$. A value of $\mathrm{R}^2 = 1$ indicates a perfect local approximation, while negative values correspond to explanations that perform worse than a constant baseline predictor. 

\inlinetitle{Stability}{.}~An essential quality of a reliable explanation method is stability: for a given instance $x$, the explanation should not change drastically due to random artifacts of the explanation process. To evaluate stability, we generate $R$ random local data perturbations of $x$, and then we produce one explanation for each of them, $E_r(x)$, $r=1,...,R$. We then calculate the stability score by measuring the average pairwise cosine similarity among these $R$ vectors. The cosine similarity is employed as it assesses the structural consistency of the explanation, \ie whether the relative ordering and direction of feature importance remain the same, regardless of minor variations in the magnitude of the weights. 

The stability score defined below by \Eq{eq:stability_formula} is normalized by the total number of unique explanation pairs, and takes values in $[0,1]$. A score closer to $1$ indicates high stability, hence the explanation is robust to the randomness inherent in the LIME sampling process:
\begin{equation}
\!\!\mathrm{Stability}(x_{i}) = \frac{1}{\binom{R}{2}} \sum_{r=1}^{R-1} \sum_{r'=r+1}^{R} \!\!\mathrm{Cosine}(E_r(x_{i}), E_{r'}(x_{i})),
\label{eq:stability_formula}
\end{equation}
where $\binom{R}{2} = \frac{R!}{2!(R-2)!}$, and $\mathrm{Cosine}(a, b) = \frac{a^\top b}{||a||_2||b||_2}$, for two column vectors $a$, $b$ of same dimension.

\inlinetitle{Regularity}{.}~Further, we can assess the \emph{smoothness} of an explanation method through regularity, which measures how 
the generated explanations vary locally in the input space, and more specifically across nearby datapoints. This can be seen as a stability score conditional to a given $x_i$, where instead of random perturbations, the test points are constrained over the real nearby data. Intuitively, if two datapoints are close to each other, their corresponding explanations
should also be similar. High regularity indicates that the surrogate model yields explanations that evolve smoothly and coherently, avoiding abrupt changes that could hinder human interpretability.

For each instance $x$ in a test set, we identify its $k$-Nearest Neighbors ($k$-NN), denoted by $\mathcal{N}_k(x)$, using the Euclidean distance in the input space. The \emph{local regularity} for $x$ is then computed as the average cosine similarity between the explanation of $x$ and those of its neighbors:
\begin{equation}
\mathrm{Regularity}_{k}(x) = \frac{1}{k} \sum_{x_i \in \mathcal{N}_k(x)} 
\mathrm{Cosine}(E(x), E(x_{i})).
\label{eq:local_regularity}
\end{equation}

This index takes values in $[0,1]$. A regularity score closer to $1$ indicates that similar inputs produce highly consistent explanations, reflecting smooth and well-structured surrogate behavior. In contrast, low regularity suggests that explanations fluctuate irregularly across neighboring points, which can undermine the interpretability and reliability of the explanation method.

\subsection{Understanding Neural Decision Trees}\label{sec:NDTs}

There exist several propositions regarding how to build an NDT, see surveys in \cite{coevolvingNDTsDTs,li2022surveyneuraltrees}. In the following, we refer to the work in \cite{Neural_Decision_Tree}, and also the work in \cite{to_tree_or_not_to_tree} that sheds light to certain properties and practical considerations of that model. 
NDTs bridge the gap between neural networks and symbolic reasoning, and most notably DTs that are rule-based models:
\begin{itemize}
    \item \emph{Differentiable architecture}: NDTs replace hard splits (\eg $\Ind{x_d \leq \beta_i}$) with \emph{soft, differentiable approximations}, enabling end-to-end training via backpropagation \cite{to_tree_or_not_to_tree}.
    \item \emph{Global optimization}: Unlike greedy DTs, NDTs optimize all splits simultaneously, potentially discovering more expressive and stable decision boundaries \cite{Neural_Decision_Tree}.
    \item \emph{Hybrid strengths}: NDTs retain the interpretability of DTs while leveraging the \emph{representation power of neural networks}, making them suitable for approximating complex models \cite{to_tree_or_not_to_tree}.
\end{itemize}

More formally, consider a DT defined over an input space $[0,1]^d$ with $K \geq 2$ leaves.  
Recent work has shown that any such DT can be represented as a three-layer neural network with two hidden layers and one output layer \cite{Neural_Decision_Tree}. This simple architecture is illustrated in \Fig{fig:dt-to-ndt}.
Let $\mathcal{H} = \{H_1, \dots, H_{K-1}\}$ be the set of hyperplanes corresponding to the internal splits of the DT, where each hyperplane is:
\begin{equation}
H_k = \{\mathbf{x} \in [0,1]^d \;|\; h_k(\mathbf{x}) = 0 \}, \quad
h_k(\mathbf{x}) = x^{(j_k)} - \alpha_{j_k},
\end{equation}
where $j_k \in \{1, \dots, d\}$ is the index of the splitting feature, and $\alpha_{j_k} \in [0,1]$ is the associated split threshold.

\inlinetitle{First hidden layer}{.}~%
This layer encodes the decisions at internal nodes.  
Let $\tau(u) = 2 \Ind{u \geq 0} - 1$ be a threshold activation function.  
The activation for neuron $k$ is:
\begin{equation}\label{eq:thresholded_activation}
z^{(1)}_k = \tau\big(x^{(j_k)} - \alpha_{j_k}\big).
\end{equation}
The weight matrix $W^{(1)} \in \mathbb{R}^{(K-1) \times d}$ is defined by:
\[
W^{(1)}_{k,m} =
\begin{cases}
1, & \text{if node $k$ splits on feature $m$},\\
0, & \text{otherwise},
\end{cases}
\]
and the bias vector $b^{(1)} \in \mathbb{R}^{K-1}$ with $b^{(1)}_k = -\alpha_{j_k}$.  
The output of this layer is a $\{\pm 1\}$ vector encoding all internal decisions.\vspace{0.4cm}

\inlinetitle{Second hidden layer}{.}~%
This layer maps the internal decision bits to leaf indicators.  
Let $\mathcal{L} = \{L_1, \dots, L_K\}$ be the set of leaves, and $l(k')$ the length of the path from root to leaf $L_{k'}$.  
We connect neuron $k$ in layer $1$ to neuron $k'$ in layer $2$ if the hyperplane $H_k$ is on the path to $L_{k'}$. The connection weight is \cite{Neural_Decision_Tree}: 
\[
W^{(2)}_{k,k'} =
\begin{cases}
+1, & \text{if $H_k$ leads to a right child on path to $L_{k'}$},\\
-1, & \text{if $H_k$ leads to a left child on path to $L_{k'}$},\\
\phantom{-}0, & \text{otherwise}.
\end{cases}
\]
The bias is $b^{(2)}_{k'} = -\frac{1}{2}(l(k') - 1)$.  
\\The activation is: $v_{k'}(\mathbf{x}) = \tau( W^{(2)T} z^{(1)} + b^{(2)} )_{k'}$, yielding a vector with the value $-1$ in all entries but for the leaf node reached by $\mathbf{x}$ for which the value is $+1$.\vspace{0.4cm}

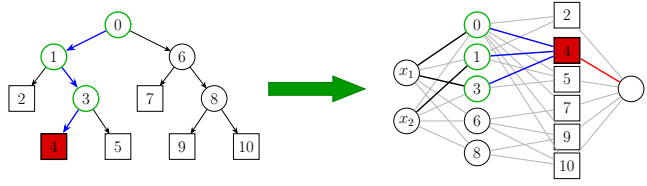
\begin{figure}[t]
\centering
\resizebox{0.7\textwidth}{!}{		
\begin{tikzpicture}[
>=Stealth,
circ/.style={circle,draw,minimum size=8mm,inner sep=0pt, font=\Large},
sq/.style={rectangle,draw,minimum size=8mm,inner sep=0pt, font=\Large},
greenC/.style={circ,draw=green!70!black,very thick},
redSq/.style={sq,fill=red!85!black,text=black,very thick, font=\Large},
blueedge/.style={blue,very thick,->},
blueline/.style={blue,very thick},
grayedge/.style={gray!60},
]

\begin{scope}

\node[greenC] (n0) at (0,2) {0};
\node[greenC] (n1) at (-2,1) {1};
\node[circ]   (n6) at (2,1) {6};

\node[sq]     (n2) at (-3,-0.3) {2};
\node[greenC] (n3) at (-1,-0.3) {3};
\node[sq]     (n7) at (1,-0.3) {7};
\node[circ]   (n8) at (3,-0.3) {8};

\node[redSq]  (n4) at (-2,-1.8) {4};
\node[sq]     (n5) at (0,-1.8) {5};
\node[sq]     (n9) at (2,-1.8) {9};
\node[sq]     (n10) at (4,-1.8) {10};

\draw[blueedge] (n0) -- (n1);
\draw[->] (n0) -- (n6);

\draw[->] (n1) -- (n2);
\draw[blueedge] (n1) -- (n3);

\draw[blueedge] (n3) -- (n4);
\draw[->] (n3) -- (n5);

\draw[->] (n6) -- (n7);
\draw[->] (n6) -- (n8);
\draw[->] (n8) -- (n9);
\draw[->] (n8) -- (n10);

\end{scope}

\draw[fill=green!60!black,draw=green!50!black]
(4.7,0.2) -- (6.7,0.2) -- (6.7,0.4)
-- (7.7,0) -- (6.7,-0.4)
-- (6.7,-0.2) -- (4.7,-0.2) -- cycle;

\begin{scope}[xshift=10cm]

\node[circ] (x1) at (-1,0.5) {$x_1$};
\node[circ] (x2) at (-1,-1) {$x_2$};

\node[greenC] (m0) at (1.2,2) {0};
\node[greenC] (m1) at (1.2,1) {1};
\node[greenC] (m3) at (1.2,0) {3};
\node[circ]   (m6) at (1.2,-1) {6};
\node[circ]   (m8) at (1.2,-2) {8};

\node[sq]    (s2) at (4,2.3) {2};
\node[redSq] (s4) at (4,1.2) {4};
\node[sq]    (s5) at (4,0.3) {5};
\node[sq]    (s7) at (4,-0.6) {7};
\node[sq]    (s9) at (4,-1.5) {9};
\node[sq]    (s10) at (4,-2.4) {10};

\node[circ] (out) at (6,0) {};

\foreach \j in {m0,m1,m3,m6,m8}
  \foreach \i in {x1,x2}
    \draw[grayedge] (\i) -- (\j);

\foreach \i/\j in {
m0/s2,
m0/s4,
m0/s5,
m0/s7,
m0/s9,
m0/s10,
m1/s2,
m1/s4,
m1/s5,
m3/s4,
m3/s5,
m6/s7,
m6/s9,
m6/s10,
m8/s9,
m8/s10}
\draw[grayedge] (\i) -- (\j);

\foreach \j in {s2,s4,s5,s7,s9,s10}
  \draw[grayedge] (\j) -- (out);

\draw[blueline] (m0) -- (s4);
\draw[blueline] (m1) -- (s4);
\draw[blueline] (m3) -- (s4);
\draw[red,very thick] (s4) -- (out);

\draw[very thick] (x1) -- (m0);
\draw[very thick] (x1) -- (m3);
\draw[very thick] (x2) -- (m1);

\end{scope}

\end{tikzpicture}
}%
\caption{Transforming a DT to its equivalent NDT.}
\label{fig:dt-to-ndt}
\end{figure}

\inlinetitle{Output layer}{.}~%
For regression, let $\bar{y}_{k'}$ be the average target value for samples in leaf $L_{k'}$. Then:
\[
W^{\text{out}}_{k'} = \frac{1}{2} \bar{y}_{k'}, \quad
b^{\text{out}} = \frac{1}{2} \sum_{k'=1}^K \bar{y}_{k'}.
\]
The prediction is:
\[
\hat{y}(\mathbf{x}) = (W^{\text{out}})^\top v(\mathbf{x}) + b^{\text{out}}.
\]
For classification with $C$ classes, let $P^c_{k'}$ be the empirical class probability in leaf $L_{k'}$. Then:
\[
W^{\text{out}}_{k',c} = \frac{1}{2} P^c_{k'}, \quad
b^{\text{out}}_{c} = \frac{1}{2} \sum_{k'=1}^K P^c_{k'},
\]
and the logits are:
\[
{f}(\mathbf{x}) = (W^{\text{out}})^\top v(\mathbf{x}) + b^{\text{out}}.
\]

\section{LIME with NDTs as surrogate models}
\label{sec:lime-ndt-main}

In \Sec{sec:NDTs},  %
we provided a comprehensive overview of Neural Decision Trees (NDTs), highlighting their unique properties and why they are particularly well-suited for generating interpretable explanations. NDTs combine the hierarchical, rule-based structure of traditional decision trees with the expressive power of neural networks, enabling them to capture complex, non-linear decision boundaries while retaining interpretability. This makes NDTs an ideal candidate for addressing the limitations of conventional surrogate models in explainable AI, particularly in scenarios where both fidelity and transparency are critical.

Next, we propose the \LIME{NDT} variant that integrates effectively NDTs as surrogate models in the LIME framework. By leveraging their differentiable architecture and global optimization capabilities, NDTs offer a robust solution for improving the local fidelity and stability of explanations. This integration not only enhances the accuracy of local approximations but also ensures that explanations remain human-interpretable and consistent across perturbations, which is a key requirement for trustworthy AI systems.

\subsection{The proposed \LIME{NDT} method}
\label{sec:lime-ndt}

To overcome the limitations of standard LIME surrogates, we propose using Neural Decision Trees (NDTs) as local surrogate models. NDTs combine the hierarchical, rule-based structure of decision trees with the expressive power of neural networks \cite{to_tree_or_not_to_tree}. Importantly, NDTs retain interpretability: for any instance, one can trace the path from root to leaf to extract explicit rules or feature contributions, and softened decision boundaries via differentiable activations do not obscure this interpretability. This ensures explanations remain both faithful to the black-box model and comprehensible to humans.

Beyond serving as a drop-in replacement for traditional LIME surrogates, the integration of NDTs introduces several methodological advantages. Unlike training an NDT from scratch, \NDTLIME initializes the neural decision tree using the structure and split parameters of a conventional decision tree trained on the local neighborhood. This \emph{warm-start} approach allows the surrogate to inherit a meaningful hierarchical partition of the input space while still permitting fine-tuning through differentiable optimization, resulting in a smoother and more stable local decision boundary.

Once initialized, the NDT is refined using gradient-based optimization to minimize the discrepancy between its predictions and those of the black-box model on the perturbed samples \cite{lime}. This optimization follows a local fidelity objective defined as:
\begin{equation}
\mathcal{L}_{\text{fidelity}} = \sum_{i=1}^{N} \pi(x_i, x)\, \| f(x_i) - g(x_i) \|^2,
\label{eq:fidelity_loss}
\end{equation}
where $\pi(x_{i}, x) = \exp(- \frac{\|x-x_{i}\|^2}{\sigma^2})
$, $\sigma$ is the kernel width that controls locality. We optimize this loss via backpropagation over the soft splits and leaf parameters of the NDT, thereby tuning the surrogate to align locally with the black-box predictions. By minimizing \(\mathcal{L}_{\text{fidelity}}\), the surrogate concentrates its representational power on the most relevant region, improving local fidelity while preserving interpretability \cite{lime}.

\begin{algorithm}[t]\small
\caption{\textbf{\NDTLIME: LIME with Neural Decision Trees}}
\label{alg:limendt}
\begin{algorithmic}
\renewcommand{\algorithmicrequire}{\textbf{Input:}}
\renewcommand{\algorithmicensure}{\textbf{Output:}}
    \Require Data instance $x$ to explain; black-box model $f$; number of perturbed samples $N$; proximity kernel $\pi$; feature set $F$.
    \Ensure Local explanation $expl$ for $x$ (\eg feature importance vector or rule set).
		
		\vspace{0.3em}
		\hrule
		\vspace{0.3em}

    \State $\mySqBullet$~\textbf{Step 1: Generate perturbed samples around the instance}
    \State \hspace{1em}$X' \gets \mathrm{Perturb}(x, N, F)$
    \Comment{Randomly perturb features of $x$} \\
		 \hfill to form a local neighborhood

    \State $\mySqBullet$~\textbf{Step 2: Obtain black-box predictions}
    \State \hspace{1em}$Y \gets f(X')$
    \Comment{Evaluate the black-box model}\\
		\hfill on all perturbed instances

    \State $\mySqBullet$~\textbf{Step 3: Compute proximity weights}
    \State \hspace{1em}$W \gets \pi(X', x)$
    \Comment{Assign higher weights to samples closer to $x$}

    \State $\mySqBullet$~\textbf{Step 4: Train the NDT surrogate model}
    \State \hspace{1em}(i) $DT_{\text{surr}} \gets \mathrm{TrainDT}(X', Y, \text{sample\_weight}=W)$
    \State \hspace{1em}(ii) $g \gets \mathrm{ConvertDTtoNDT}(DT_{\text{surr}})$
    \Comment{Initialize NDT ${\color{red}{\star}}$}
    \State $\mySqBullet$~\textbf{Step 5: Extract interpretable local explanation}
    \State \hspace{1em}$expl \gets \mathrm{ExtractRules}(g, x, F)$
    \Comment{Derive feature importance}\\ 
		\hfill or rule-based explanation for $x$

    \State \Return $expl$
\end{algorithmic}
\hrule\leavevmode \\
${\color{red}{\star}}$ {\footnotesize Initialize NDT structure and parameters (splits, thresholds, leaf values) using an input trained DT.}
\end{algorithm}

\NDTLIME is summarized in \Alg{alg:limendt}, which outlines the main steps from generating perturbed samples to training the NDT surrogate and extracting the final local explanation.
The pseudocode above summarizes the \NDTLIME pipeline. The conversion from a decision tree to an NDT allows leveraging the interpretability of rule-based structures while enabling continuous optimization in the local neighborhood of the explained instance. In practice, this process adds only a modest computational overhead compared to traditional LIME but yields explanations that are smoother, more faithful to the underlying model.

\subsection{Properties of NDTs as LIME surrogate models}
\label{sec:properties-lime-ndt}

\inlinetitle{The NDT of choice}{.}~%
As already mentioned in \Sec{sec:NDTs}, there are several NDT architectures. Our choice of NDT model is driven by its unique properties that are particularly well-suited for generating interpretable local explanations.

A key property we leverage is regularity, which stems from the use of soft, differentiable splits. Unlike the rigid, axis-aligned partitions of classical decision trees, these soft splits promote smoother transitions between local regions of the input space. This allows \NDTLIME to provide explanations that are both faithful to the complex model and continuous across the input manifold. As a result, the explanations gain enhanced human interpretability and coherence, as minor changes in the input data do not lead to abrupt, counter-intuitive shifts in the explanation. Technically, this smoothness is achieved by replacing the hard threshold function $\tau(u)$ from Eq.~\eqref{eq:thresholded_activation}, which is non-differentiable and thus prevents gradient-based optimization, with a smooth activation function~\cite{to_tree_or_not_to_tree}:
\[
\sigma_{\gamma}(u) = \tanh(\gamma u), \quad \gamma > 0,
\]
where the parameter $\gamma$ controls the sharpness of the decision boundaries. 
Prior work in~\cite{to_tree_or_not_to_tree} has investigated in detail how these smoothing parameters affect the flexibility of NDTs.
Specifically, it was shown that large $\gamma$ values approximate hard, tree-like splits, whereas smaller $\gamma$ values produce smoother, neural-network-like transitions.
Furthermore, an analysis of the relationship between the parameters of the first and second layers, $\gamma_1$ and $\gamma_2$, led to the suggestion that $\gamma_1$ should be greater than or equal to $\gamma_2$ ($\gamma_1 \geq \gamma_2$). 
This guideline helps preserve meaningful split activations at the decision nodes while enabling smoother, more adaptive leaf-level predictions, and it is a handy rule of thumb when exploring the parametrization space for NDTs.

\inlinetitle{Plotting the decision boundaries}{.}~
To qualitatively assess the smoothness induced by NDTs, we visualize the local decision boundaries of three surrogate models: a standard DT, an NDT initialized from the DT without further refinement, and a fine-tuned NDT.
We use a subset of two features and two classes from the Wine dataset. %
A single instance is selected as the point of interest, around which $R=800$ perturbed samples are generated to form a local neighborhood. The labels of these perturbed samples are obtained from a Random Forest trained on the full dataset, which serves as the black-box model to be explained. Each surrogate model is then trained on this locally generated dataset. %
For each point in the grid, predictions are computed using the corresponding surrogate model, allowing us to render their respective decision regions.

As illustrated in \Fig{Decision_boundaries}, the DT exhibits sharp and axis-aligned decision boundaries that reflect the rigidity of hard splits. In contrast, the NDT initialized from that DT, before refinement, produces poorly structured and irregular decision regions. This is a direct consequence of replacing hard threshold functions with soft activations without subsequent refinement. %
Finally, through gradient-based optimization, the fine-tuned NDT adapts its soft splits and leaf parameters to the local neighborhood, transforming the initially degraded boundary into a coherent and faithful approximation. Eventually, NDT recovers well-shaped and smooth decision boundaries that follow the local behavior of the black-box model. 

\begin{figure}[t]
    \centering
    \includegraphics[width=\columnwidth]{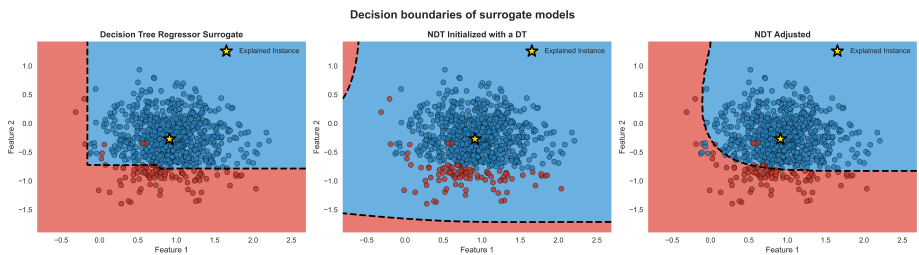}
    \caption{Local decision boundaries of surrogate models on the Wine dataset (a subset of $2$ features and $2$ classes). From left to right: a standard DT, an NDT initialized from the DT without fine-tuning, and a fine-tuned NDT. While the DT exhibits rigid, axis-aligned partitions, the non-fine-tuned NDT produces degraded and poorly structured decision regions due to soft splits applied without optimization. Fine-tuning the NDT restores coherent and smooth decision boundaries, yielding a faithful approximation of the black-box model’s local decision boundary.}
    \label{Decision_boundaries}
\end{figure}
\begin{figure}[t]
    \centering
    \includegraphics[width=0.7\columnwidth]{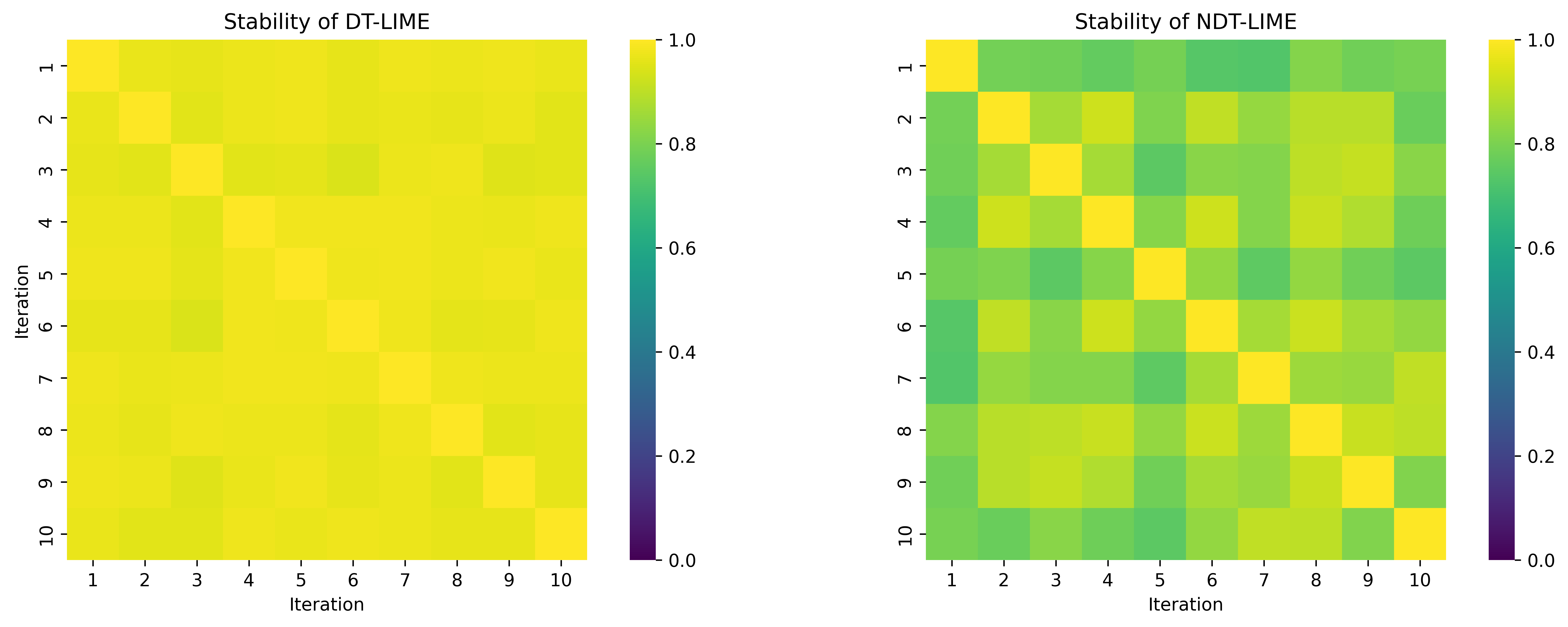}
    \caption{Stability matrices of DT-LIME (left) and NDT-LIME (right) explanations for a fixed instance from the Digits dataset.
}
    \label{stability_matrix}
\end{figure}

\inlinetitle{Continuity of explanations}{.}~%
Let $g_{\mathrm{NDT}} : \mathbb{R}^d \to \mathbb{R}$ denote an NDT surrogate. We define the associated explanation map as:
\begin{equation}
E : \mathbb{R}^d \to \mathbb{R}^d,
\qquad
E(x) := \nabla g_{\mathrm{NDT}}(x),
\label{eq:explanation_map}
\end{equation}
where $\nabla g_{\mathrm{NDT}}(x)$ is the gradient of the surrogate with respect to the input features.

\begin{proposition}[Continuity of explanations]
As $g_{\mathrm{NDT}} \in C^1(\mathbb{R}^d)$ (formally established in Appendix~\ref{NDT_C_infinite}), then the explanation map $E$ defined in~\Eq{eq:explanation_map}, belongs to $C^0(\mathbb{R}^d)$. In particular, for any $x \in \mathbb{R}^d$ and for any $\varepsilon > 0$, there exists $\delta > 0$ such that:\ \ %
$\|x - x'\| < \delta
\quad \Rightarrow \quad
\|E(x) - E(x')\| < \varepsilon
$.
\end{proposition}

\begin{proof}
Since $g_{\mathrm{NDT}} \in C^1(\mathbb{R}^d)$, its gradient exists and is continuous
on $\mathbb{R}^d$.
Therefore, the mapping $x \mapsto \nabla g_{\mathrm{NDT}}(x)$ is continuous,
which implies $E \in C^0(\mathbb{R}^d)$.
\end{proof}

The continuity of the explanation map ensures that small perturbations of the input produce small variations in the resulting explanations. This provides a formal notion of \emph{input-level stability} for NDT-based explanations, which is not guaranteed for non-differentiable surrogate models such as hard decision trees.

\inlinetitle{Stability of NDTs}{.}~%
We analyze the stability of explanations produced for a fixed input instance from the Digits dataset by repeating the surrogate training procedure with different random initializations.
The resulting stability matrix illustrated in \Fig{stability_matrix}, summarizes how feature attributions vary across runs.

This analysis captures algorithmic variability in the explanation process, arising from factors such as random sampling, model initialization, and non-convex optimization.
It is therefore complementary to the theoretical continuity guarantees established for NDTs, which concern stability with respect to input perturbations for a fixed surrogate model.
The observed variability reflects a fundamental bias-variance tradeoff in explanation models. Linear surrogates exhibit highly stable but overly simplistic explanations, while more expressive models may yield explanations that vary across runs despite improved fidelity. NDTs occupy an intermediate regime, combining smooth and expressive explanations with moderate variability.
The stability matrix provides an empirical perspective on explanation robustness that complements theoretical regularity results.

\section{Experiments} \label{sec:exps}

The experiments are focused on tabular data and conduct a series of experiments on eight benchmark datasets: Breast Cancer%
, Iris%
, Wine%
, Digits%
, Covertype%
, Diabetes%
, California Housing%
, and Ames Housing%
. 
The details of the datasets are provided in \Tab{tab:dataset_details} in Appendix~\ref{app:dataset-details}. \Tab{tab:dataset-list-literature} in the Appendix, puts in perspective the range of datasets used in this work and how it compares to related works.

The empirical results evaluating the effectiveness of the proposed \NDTLIME method are reported in \Tab{tab:stability-fidelity}. We compare against two baselines: the standard LIME using a linear regression surrogate (\LIME{LR}), and LIME using a traditional greedy DT regressor as surrogate (\LIME{DT}). Additional comparisons with several LIME variants from the literature are provided in \Fig{fig:lime_comparison} and in the Appendix~\ref{app:implementation-competitors} to further contextualize the performance of the proposed approach. For all experiments, the black-box model to be explained is a neural network with four hidden layers of sizes $512$, $256$, $128$, and $64$ neurons, which allows us to evaluate the interpretability of \NDTLIME on a complex, non-linear model.
In all experiments, the NDT's smoothing parameters were set to $\gamma_{1} = 1$ for the first layer and $\gamma_{2} = 1$ for the second layer, following the guideline of \cite{to_tree_or_not_to_tree} that $\gamma_{1}$ should be greater than or equal to $\gamma_{2}$ ($\gamma_{1} \geq \gamma_{2}$) to preserve meaningful split activations while allowing smoother leaf-level transitions.

An open source repository with implementations for the proposed \NDTLIME and the compared methods is publicly available\footnote{Accessible at: 
\url{https://github.com/aymen20002005/lime_ndt}
}.

\subsection{Comparing \NDTLIME to existing LIME Variants}

\begin{table*}[t]
\centering
\resizebox{0.8\textwidth}{!}{
\centerline{
  \begin{tabular}{l|ccc|ccc|ccc}
    \toprule
		\multirow{2}{*}{\textbf{Dataset}} & \multicolumn{3}{c}{\textbf{Stability}\ $\uparrow$}  & \multicolumn{3}{c}{\textbf{Fidelity}\ $\uparrow$} & \multicolumn{3}{c}{\textbf{Regularity}\ $\uparrow$}\\
    & \LIME{LR} & \LIME{DT} & \NDTLIME & \LIME{LR} & \LIME{DT} & \NDTLIME & \LIME{LR} & \LIME{DT} & \NDTLIME\\
    \midrule
    Breast Cancer & \textbf{0.997$\pm$0.003} & 0.986$\pm$0.103 & \textit{0.991$\pm$0.004} & 0.527$\pm$0.085 & \textit{0.686$\pm$0.049} & \textbf{0.785$\pm$0.031} & 0.812$\pm$0.014 & \textit{0.873$\pm$0.005} & \textbf{0.915$\pm$0.024}\\
    Iris & \textit{0.994$\pm$0.006} & \textbf{0.997$\pm$0.003} & 0.943$\pm$0.010 & 0.554$\pm$0.150 & \textit{0.777$\pm$0.037} & \textbf{0.860$\pm$0.021} & 0.743$\pm$0.017 & \textit{0.813$\pm$0.028} & \textbf{0.820$\pm$0.039}\\
    Wine & \textbf{0.999$\pm$0.001} & 0.997$\pm$0.001 & \textit{0.998$\pm$0.002} & 0.321$\pm$0.133 & \textit{0.395$\pm$0.183} & \textbf{0.518$\pm$0.131} & 0.830$\pm$0.019 & \textit{0.864$\pm$0.042} & \textbf{0.910$\pm$0.023}\\
    Digits & \textbf{0.980$\pm$0.009} & \textit{0.961$\pm$0.091} & 0.816$\pm$0.023 & 0.243$\pm$0.141 & \textit{0.440$\pm$0.076} & \textbf{0.577$\pm$0.106} & \textbf{0.654$\pm$0.021} & 0.512$\pm$0.084 & \textit{0.563$\pm$0.034}\\
    Covtype & \textbf{0.984$\pm$0.005} & \textit{0.983$\pm$0.108} & 0.931$\pm$0.007 & 0.362$\pm$0.052 & \textit{0.556$\pm$0.110} & \textbf{0.632$\pm$0.067} & \textit{0.890$\pm$0.012} & \textbf{0.926$\pm$0.020} & 0.849$\pm$0.072\\
	\midrule
    California Housing & \textbf{0.999$\pm$0.000} & \textbf{0.999$\pm$0.000} & \textit{0.973$\pm$0.001} & 0.297$\pm$0.148 & \textit{0.890$\pm$0.022} & \textbf{0.960$\pm$0.014} & \textit{0.834$\pm$0.009} & \textbf{0.864$\pm$0.011} & 0.794$\pm$0.033\\
    Diabetes & \textbf{0.999$\pm$0.001} & \textbf{0.999$\pm$0.001} & \textit{0.998$\pm$0.002} & \textit{0.886$\pm$0.012} & 0.562$\pm$0.124 & \textbf{0.920$\pm$0.035} & \textit{0.956$\pm$0.010} & 0.794$\pm$0.102 & \textbf{0.978$\pm$0.017}\\
    Ames Housing & \textbf{0.998$\pm$0.000}& \textit{0.996$\pm$0.041} & 0.990$\pm$0.013 & \textbf{0.865$\pm$0.030} & 0.506$\pm$0.062 & \textit{0.713$\pm$0.054} & \textbf{0.976$\pm$0.008} & \textit{0.903$\pm$0.012} & 0.893$\pm$0.027\\
    \bottomrule
  \end{tabular}%
}
}
\vspace{1em}
\caption{Stability, Fidelity and Regularity for \LIME{LR}, \LIME{DT}, and \NDTLIME methods on benchmark datasets. The stability of local explanations is measured as the average pairwise cosine similarity across repeated runs with random data perturbations. The fidelity of the involved surrogate models (\ie LR, DT, NDT) is measured as the $\mathrm{R}^2$ score between the surrogate predictions and the black-box model predictions on perturbed neighborhood samples. The regularity is computed as the average cosine similarity between the explanation of each test instance and the explanations of its $k$-NNs (with $k=2$), measuring how smoothly explanations evolve across nearby instances. Higher values indicate better approximation of the black-box model.
}
\label{tab:stability-fidelity}
\vspace{-2em}
\end{table*}

\inlinetitle{Measuring Stability}{.}~%
\Tab{tab:stability-fidelity} shows that all three surrogate models produce highly stable explanations, with the vast majority of scores exceeding $0.95$. The traditional \LIME{LR} and \LIME{DT} surrogates consistently yield extremely high stability scores, often approaching a perfect $1.0$. This is expected, as their simpler structures are less sensitive to minor variations in the local data sample, at the expense of faithfully modeling complex decision boundaries.

The proposed \NDTLIME method also demonstrates strong stability across the benchmark. On datasets such as Breast Cancer ($0.991$), Diabetes ($0.998$), Wine ($0.998$), and Ames Housing ($0.990$), stability scores are nearly indistinguishable from the baselines. Slightly lower stability is observed on California Housing ($0.973$), Iris ($0.943$), Digits ($0.816$), and Covtype ($0.931$).
This minor or small decrease in stability is not a sign of inferiority, but reflects the greater expressive power of the NDT. The gradient-based, global optimization of the NDT surrogate can capture subtle non-linear patterns in each random perturbation of the local data. As a result, explanations may vary slightly across runs while remaining valid. In fact, having slightly lower stability introduces more variance in the explanations, which can increase the chance of discovering the most faithful or insightful explanation for the underlying black-box model.

This is a critical distinction, as most other LIME variants shown in \Fig{fig:lime_comparison} retain a linear surrogate and instead focus on modifying the sampling or weighting process. Their near-perfect stability is often a byproduct of the surrogate's inflexibility. By fundamentally changing the surrogate model to an NDT, our method achieves a more meaningful balance between stability and expressiveness, as visually confirmed in the trade-off plots of \Fig{fig:lime_comparison}.

\inlinetitle{Measuring Fidelity}{.}~%
\label{Measuring_Fidelity}
As shown in \Tab{tab:stability-fidelity}, \NDTLIME consistently achieves higher fidelity scores compared to both \LIME{LR} and \LIME{DT} across most datasets. On California Housing and Diabetes, the Neural Decision Tree surrogate significantly outperforms the linear and greedy decision tree surrogates, reaching $\mathrm{R}^2$ scores above $0.9$. For Iris and Breast Cancer datasets, \NDTLIME still surpasses the other surrogates, though the fidelity remains below $0.9$, reflecting moderate local approximation accuracy. This indicates that NDTs can better capture non-linear local decision boundaries than traditional surrogates, even if the approximation is not perfect. For instance, on California Housing, \LIME{LR} achieves a modest fidelity of $0.297$, reflecting its limited ability to approximate highly non-linear patterns. In contrast, \NDTLIME reaches a nearly perfect fidelity ($0.960$), while maintaining interpretability through its tree-like structure. Similarly, for the Diabetes, \LIME{LR} performs well ($0.886$), but \NDTLIME remains competitive ($0.920$) and surpasses the decision tree surrogate ($0.562$).

The Wine, Covtype and Digits datasets represent challenging scenarios where all surrogates achieve relatively low fidelity, with $\mathrm{R}^2$ scores below $0.7$. This suggests that the Neural Networks’ decision boundaries in these datasets are intricate and difficult for simple surrogates to capture locally. Nonetheless, \NDTLIME performs better than \LIME{LR} and \LIME{DT}, highlighting its robustness even in difficult cases.
Overall, these results indicate that \NDTLIME provides a strong balance between interpretability and fidelity. By leveraging the expressive power of Neural Decision Trees, it achieves explanations that are generally faithful to the black-box model and more representative of its local decision boundaries compared to traditional surrogates. This performance gain stems directly from our choice of surrogate model, which is the most critical factor for achieving high fidelity. As illustrated in \Fig{fig:lime_comparison}, many competing LIME variants fail to reach high fidelity scores precisely because they are constrained by a simple linear surrogate. By replacing this core component with a more powerful NDT, NDT-LIME fundamentally overcomes this limitation, allowing it to capture complex, non-linear behaviors where other methods fail.

To further evaluate the robustness of \NDTLIME, we also conducted experiments to assess how local fidelity is impacted by the complexity of the global black-box model. We systematically increased the complexity of the Neural Network used as black-box model and measured the fidelity of \NDTLIME. The results, detailed in Appendix~\ref{sec:fidelity_complexity}, demonstrate that \NDTLIME maintains significantly higher local fidelity as the complexity of the black-box model increases.

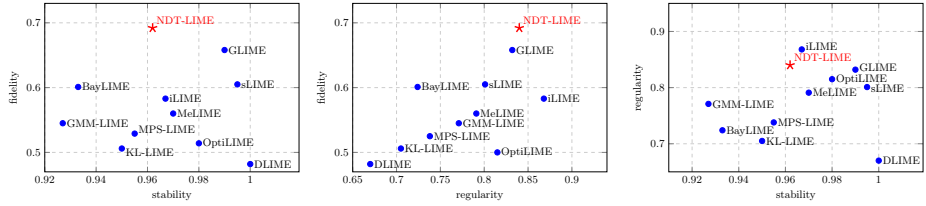
\begin{figure*}[t!]
\centering
\begin{minipage}[t]{0.32\textwidth}
\centering
\resizebox{\linewidth}{!}{%
\begin{tikzpicture}
    \begin{axis}[
        width=9cm,
        height=6.5cm,
        xlabel={stability},
        ylabel={fidelity},
        xmin=0.92, xmax=1.019,
        ymin=0.47, ymax=0.73,
        grid=both,
        major grid style={dashed, gray!40},
        minor grid style={dotted, gray!20},
        tick label style={font=\small},
        label style={font=\small},
        legend style={at={(0.03,0.97)}, anchor=north west, draw=none, fill=white, font=\small},
    ]
    \addplot[only marks, mark=star, mark size=4pt, very thick, color=red] coordinates {(0.962, 0.692)};
    \node[above right, font=\small, text=red] at (axis cs:0.962,0.692) {NDT-LIME};

    \addplot[only marks, mark=*, mark size=2.2pt, color=blue] coordinates {
        (0.990, 0.658) (1, 0.482) (0.955, 0.529) (0.95, 0.506)
        (0.970, 0.560) (0.995, 0.605) (0.967, 0.583)
        (0.933, 0.601) (0.980, 0.514) (0.927, 0.545)
    };

    \node[right] at (axis cs:0.990,0.658) {\small GLIME};
    \node[right] at (axis cs:1,0.482) {\small DLIME};
    \node[right] at (axis cs:0.955,0.535) {\small MPS-LIME};
    \node[right] at (axis cs:0.95,0.500) {\small KL-LIME};
    \node[right] at (axis cs:0.970,0.560) {\small MeLIME};
    \node[right] at (axis cs:0.995,0.605) {\small sLIME};
    \node[right] at (axis cs:0.967,0.583) {\small iLIME};
    \node[right] at (axis cs:0.933,0.601) {\small BayLIME};
    \node[right] at (axis cs:0.980,0.514) {\small OptiLIME};
    \node[right] at (axis cs:0.927,0.545) {\small GMM-LIME};
    \end{axis}
\end{tikzpicture}%
}
\end{minipage}%
\hfill
\begin{minipage}[t]{0.32\textwidth} %
\centering
\resizebox{\linewidth}{!}{%
\begin{tikzpicture}
\begin{axis}[
    width=9cm,
    height=6.5cm,
    xlabel={regularity},
    ylabel={fidelity},
    xmin=0.65, xmax=0.94,
    ymin=0.47, ymax=0.73,
    grid=both,
    major grid style={dashed, gray!40},
    minor grid style={dotted, gray!20},
    tick label style={font=\small},
    label style={font=\small},
    legend style={at={(0.03,0.97)}, anchor=north west, draw=none, fill=white, font=\small},
]
\addplot[only marks, mark=star, mark size=4pt, very thick, color=red]
coordinates {(0.840, 0.692)};
\node[above right, font=\small, text=red] at (axis cs:0.840,0.692) {NDT-LIME};

\addplot[only marks, mark=*, mark size=2.2pt, color=blue] coordinates {
    (0.832, 0.658) (0.67, 0.482) (0.738, 0.525) (0.705, 0.506)
    (0.791, 0.560) (0.801, 0.605) (0.868, 0.583)
    (0.724, 0.601) (0.815, 0.5) (0.771, 0.545)
};

\node[right] at (axis cs:0.832,0.658) {\small GLIME};
\node[right] at (axis cs:0.67,0.482) {\small DLIME};
\node[right] at (axis cs:0.738,0.525) {\small MPS-LIME};
\node[right] at (axis cs:0.705,0.506) {\small KL-LIME};
\node[right] at (axis cs:0.791,0.560) {\small MeLIME};
\node[right] at (axis cs:0.801,0.605) {\small sLIME};
\node[right] at (axis cs:0.868,0.583) {\small iLIME};
\node[right] at (axis cs:0.724,0.601) {\small BayLIME};
\node[right] at (axis cs:0.815,0.500) {\small OptiLIME};
\node[right] at (axis cs:0.771,0.545) {\small GMM-LIME};

\end{axis}
\end{tikzpicture}%
}
\end{minipage}
\hfill
\begin{minipage}[t]{0.32\textwidth} %
\centering
\resizebox{\linewidth}{!}{%
\begin{tikzpicture}
    \begin{axis}[
        width=9cm,
        height=6.5cm,
        xlabel={stability},
        ylabel={regularity},
        xmin=0.91, xmax=1.019,
        ymin=0.65, ymax=0.95,
        grid=both,
        major grid style={dashed, gray!40},
        minor grid style={dotted, gray!20},
        tick label style={font=\small},
        label style={font=\small},
        legend style={at={(0.03,0.97)}, anchor=north west, draw=none, fill=white, font=\small},
    ]
    \addplot[only marks, mark=star, mark size=4pt, very thick, color=red] coordinates {(0.962, 0.840)};
    \node[above right, font=\small, text=red] at (axis cs:0.962,0.840) {NDT-LIME};

    \addplot[only marks, mark=*, mark size=2.2pt, color=blue] coordinates {
        (0.990, 0.832) (1, 0.67) (0.955, 0.738) (0.95, 0.705)
        (0.970, 0.791) (0.995, 0.801) (0.967, 0.868)
        (0.933, 0.724) (0.980, 0.815) (0.927, 0.771)
    };

    \node[right] at (axis cs:0.990,0.837) {\small GLIME};
    \node[right] at (axis cs:1,0.67) {\small DLIME};
    \node[right] at (axis cs:0.955,0.738) {\small MPS-LIME};
    \node[right] at (axis cs:0.95,0.705) {\small KL-LIME};
    \node[right] at (axis cs:0.970,0.791) {\small MeLIME};
    \node[right] at (axis cs:0.995,0.798) {\small sLIME};
    \node[right] at (axis cs:0.967,0.874) {\small iLIME};
    \node[right] at (axis cs:0.933,0.722) {\small BayLIME};
    \node[right] at (axis cs:0.980,0.815) {\small OptiLIME};
    \node[right] at (axis cs:0.927,0.771) {\small GMM-LIME};
    \end{axis}
\end{tikzpicture}%
}
\end{minipage}%
\caption{Visual summary of the results comparing \LIME{NDT} with other LIME variants in terms of stability-vs-fidelity, regularity-vs-fidelity, and stability-vs-regularity. The results in tabilar format are presented in Tabs.\,\ref{tab:fidelity_lime_variants}-\ref{tab:regularity_lime_variants}.}
\label{fig:lime_comparison}
\end{figure*}

\inlinetitle{Measuring Regularity}{.}~%
As defined in \Sec{sec:quality-measures}, regularity evaluates whether similar inputs yield similar explanation vectors, an essential property for human interpretability and trust.
For this experiment, we compute the regularity using a $k$-NN approach. For each test instance $x_{i}$, we identify its $k$ closest neighbors in the original feature space using Euclidean distance, and compute the cosine similarity between the explanation vector $E(x_{i})$ and those of its neighbors, as formalized in \Eq{eq:local_regularity}. We choose $k=2$ to focus the evaluation on very local coherence, ensuring that the metric captures genuine smoothness rather than global averaging effects. This is particularly important in the context of local explanation methods such as LIME, whose objective is to faithfully describe model behavior in a tightly constrained neighborhood. A more detailed exploration of how regularity behaves across different neighborhood sizes (\ie for varying values of k) is provided in the Appendix~\ref{Analysis of Regularity in Different Local Regions}.

The regularity results in \Tab{tab:stability-fidelity}, reveal several notable trends. \NDTLIME achieves the highest or near-highest regularity scores on the majority of datasets, indicating that its explanations vary more smoothly across neighboring datapoints than those produced by traditional surrogates.
On Breast Cancer, Wine, and Diabetes, \NDTLIME exhibits clear gains in regularity, reaching scores of $0.915$, $0.910$, and $0.978$ respectively. These improvements suggest that the soft, differentiable splits of the NDTs promote coherent transitions between nearby explanations, in contrast to the potentially abrupt changes induced by hard, axis-aligned splits in DTs %
or the overly rigid structure of linear models.
For California Housing, \NDTLIME achieves slightly lower regularity ($0.794$) compared to \LIME{LR} ($0.834$) and \LIME{DT} ($0.864$). This reflects a known trade-off: while NDTs capture more complex local non-linearities leading to very high fidelity, this expressiveness can introduce subtle variations in explanation vectors even among very close neighbors. Importantly, this does not indicate incoherent behavior, but rather a more faithful adaptation to local curvature in the black-box model’s decision boundaries.

\section{Conclusion}
In this work, we introduced \NDTLIME, a novel approach that enhances the LIME methodology by integrating Neural DTs as surrogates. Through extensive experiments on multiple benchmark tabular datasets, we demonstrated that \NDTLIME achieves a strong balance between stability and fidelity.

Overall, the findings confirm that NDTs are powerful surrogates for enhancing LIME, providing explanations that are at once interpretable, stable, and faithful. This synergy bridges the gap between black-box predictive accuracy and transparent decision-making, opening new perspectives for explainable AI in critical domains such as healthcare, finance, and policy-making.

As highlighted in the recent work in \cite{which_lime}, ensuring the trustworthiness and consistency of LIME-based explanations remains a central challenge for the eXplainable AI community. Future work may focus on extending \NDTLIME to other types of data beyond tabular formats, such as images or text, as well as investigating hybrid strategies that combine local explanations with global interpretability frameworks. By doing so, we aim to further strengthen the trustworthiness and applicability of machine learning in real-world, high-stakes scenarios.

\bibliographystyle{IEEEtran}
\input{main.bbl}

\newpage
\section*{APPENDIX}

\section{Analysis of Local Surrogate Fidelity Against Global Model Complexity}\label{sec:fidelity_complexity}

To demonstrate the robustness and necessity of the NDT as a surrogate model, we investigated how local explanation fidelity is impacted by the complexity of the global black-box model. We trained a series of Multi-Layer Perceptrons (MLPs) on five benchmark datasets, systematically increasing complexity by adding hidden layers. Fidelity is measured using the $\mathrm{R}^2$ score of the local surrogate model's predictions against the complex model's predictions in the local neighborhood. \Fig{fig:all_datasets}, presented as eight subplots, illustrates the fidelity of \LIME{LR}, \LIME{DT}, and \NDTLIME as the global model complexity increases for five distinct datasets. Across all datasets, a clear and consistent pattern emerges: \NDTLIME maintains significantly higher local fidelity as the complexity of the black-box model increases, while the fidelity of traditional surrogates (\LIME{LR} and \LIME{DT}) declines drastically or collapses.

In all figures, \LIME{LR} shows the steepest decline in fidelity. For models with $3$ or more hidden layers, the $\mathrm{R}^2$ score often drops below $0.2$, or even approaches zero. This confirms that linear models are fundamentally incapable of approximating the highly non-linear, multi-dimensional decision boundaries created by deep, multi-layer networks. The local explanation derived from \LIME{LR} quickly becomes unfaithful and untrustworthy as the global model gains complexity.

\LIME{DT} performs marginally better than \LIME{LR} initially, as its hierarchical structure offers more flexibility than a simple hyperplane. However, its fidelity rapidly plateaus or drops off after the model reaches moderate complexity ($2$-$3$ layers). The rigid, axis-aligned splits of traditional decision trees cannot capture the smooth, curved, or high-dimensional manifolds generated by deeper MLPs, causing its local fidelity to fail to keep pace with the increasing complexity of the target black-box function.

\NDTLIME consistently outperforms both baseline surrogates. Even when the global MLP model reaches maximum complexity, \NDTLIME maintains a strong $\mathrm{R}^2$ score. This sustained performance is directly attributable to the core innovation of NDTs: the use of soft, differentiable splits. These soft splits allow the NDT to smoothly approximate complex decision boundaries without the limitations imposed by the rigid axis-aligned cuts of traditional decision trees, ensuring the local explanation remains faithful even for the most complex black-box models.
\begin{figure*}[t]
    \centering
			  \newcommand{\subfigsize}{0.24}
    \centering    
    \begin{subfigure}{\subfigsize\textwidth}
			\centering
      \includegraphics[width=\textwidth]{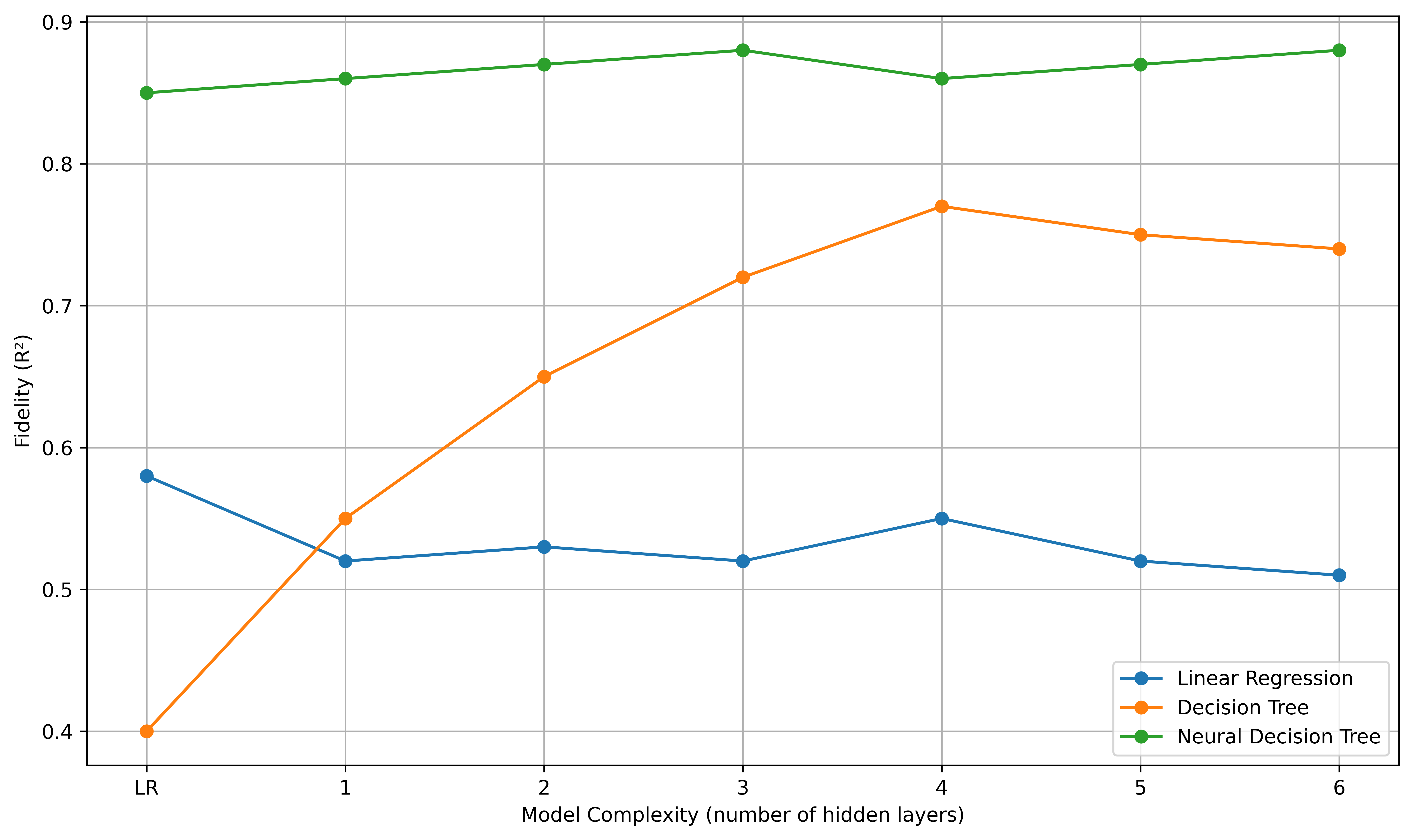}
			\caption{Iris}
			\label{subfig:Fidelity-Iris}
		\end{subfigure}%
    \hfill%
    \begin{subfigure}{\subfigsize\textwidth}
			\centering
      \includegraphics[width=\textwidth]{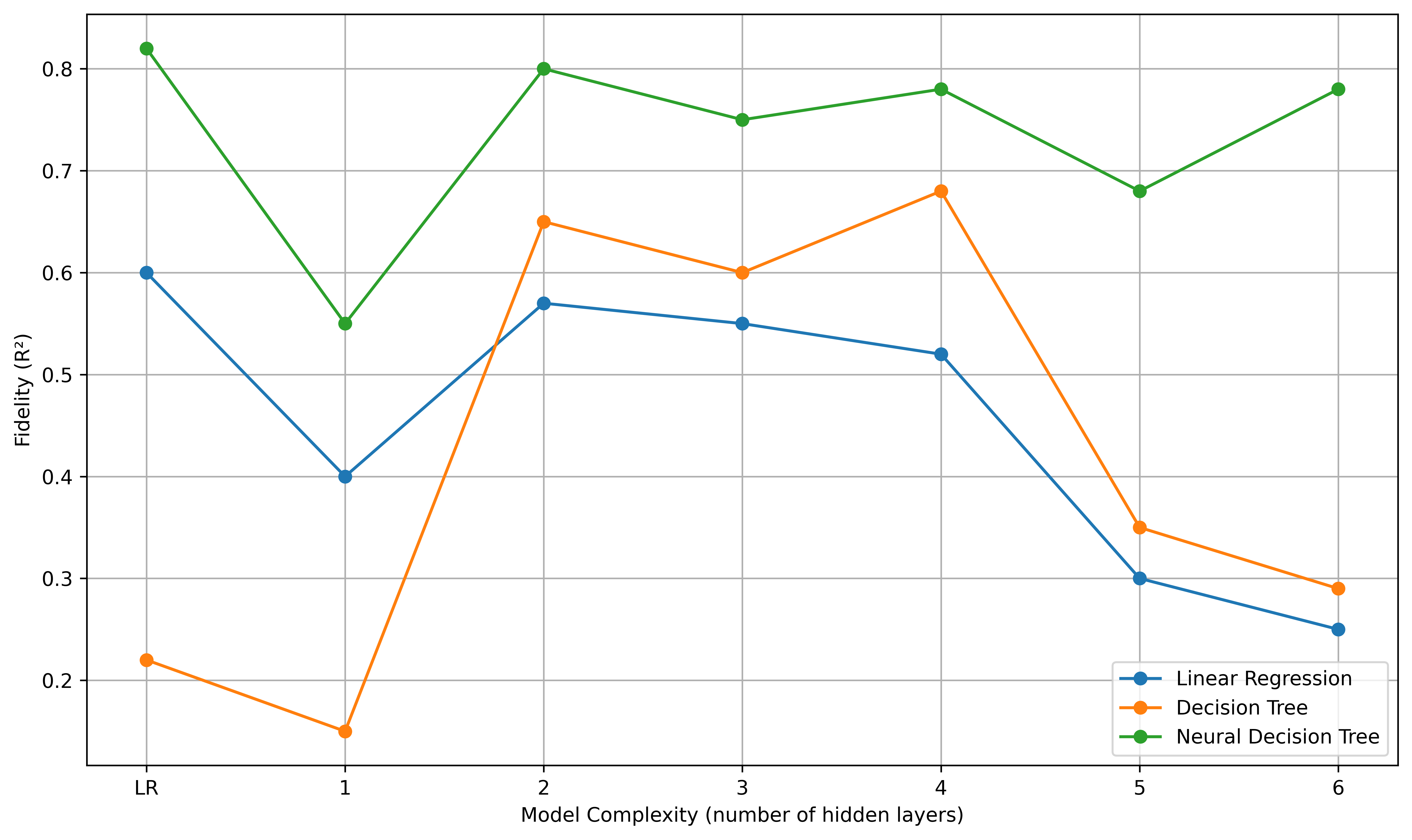}
			\caption{Breast Cancer}
			\label{subfig:Fidelity-breast_cancer}
		\end{subfigure}%
		\hfill
    \begin{subfigure}{\subfigsize\textwidth}
			\centering
      \includegraphics[width=\textwidth]{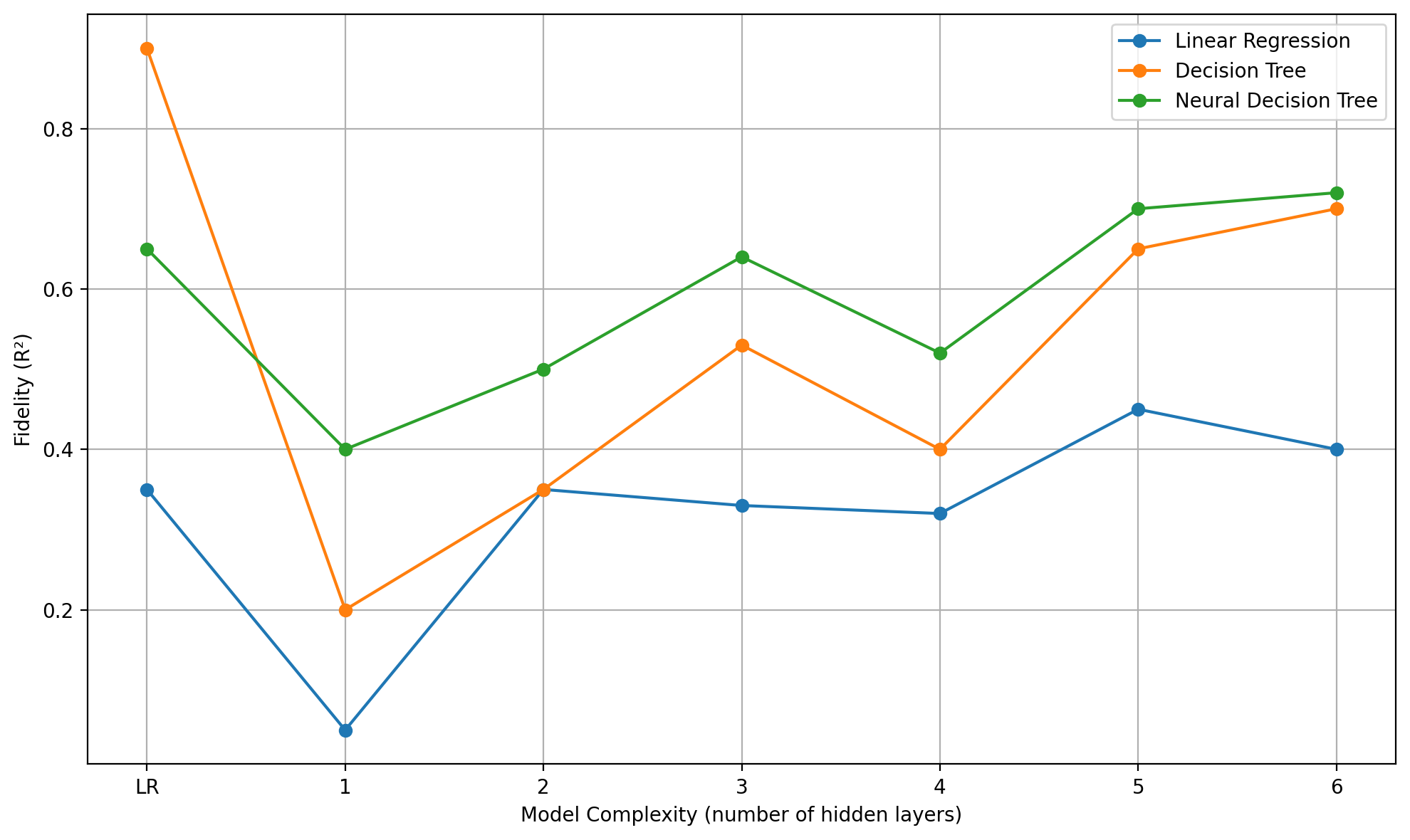}
			\caption{Wine}
			\label{subfig:Fidelity-wine}
		\end{subfigure}%
			\hfill
    \begin{subfigure}{\subfigsize\textwidth}
			\centering
      \includegraphics[width=\textwidth]{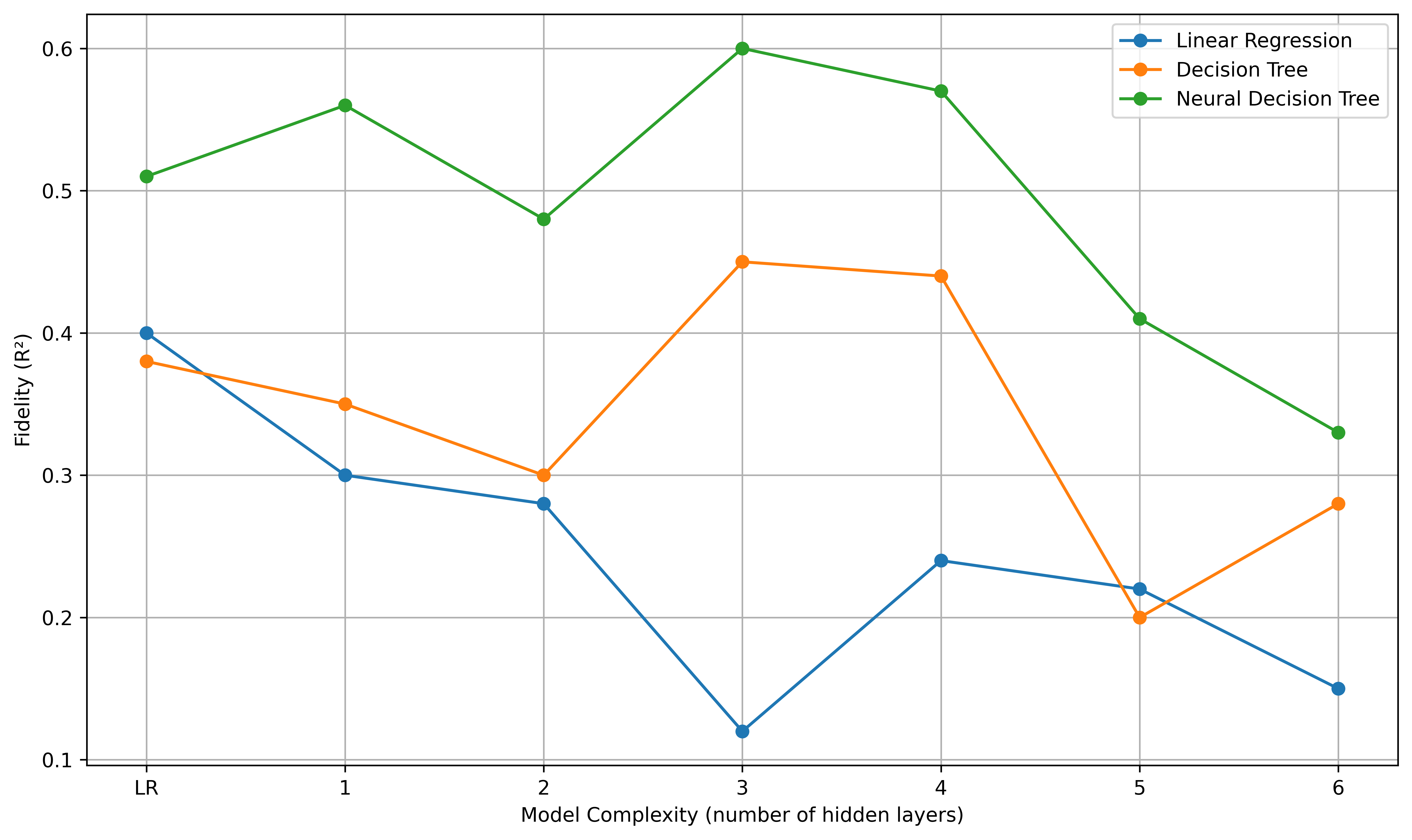}
			\caption{Digits}
			\label{subfig:Fidelity-digits}
		\end{subfigure}%
		
    \vspace{1em} %
    
		\begin{subfigure}{\subfigsize\textwidth}
			\centering
      \includegraphics[width=\textwidth]{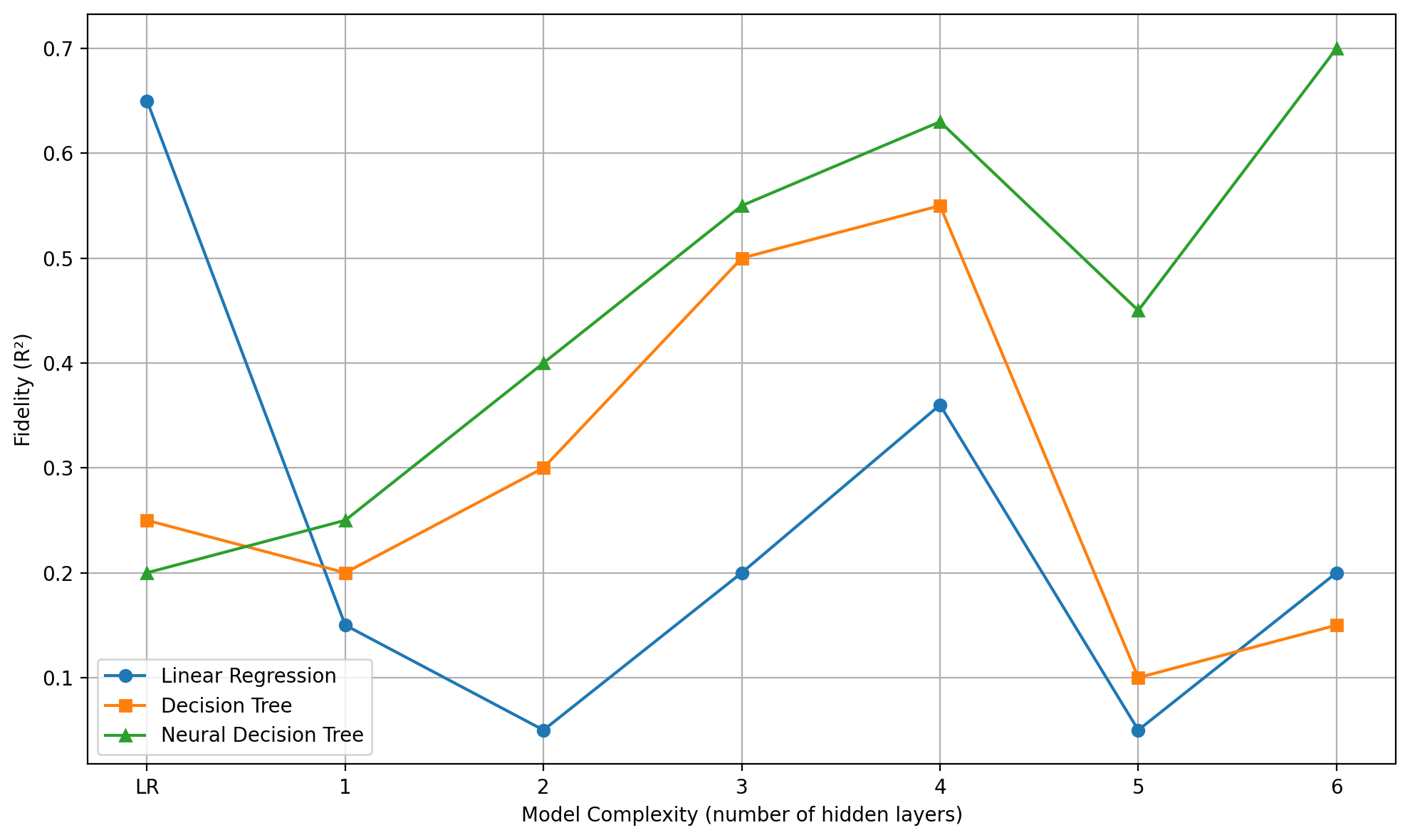}
			\caption{Covertype}
			\label{subfig:Fidelity-covertype}
		\end{subfigure}%
		\hfill
		\begin{subfigure}{\subfigsize\textwidth}
			\centering
      \includegraphics[width=\textwidth]{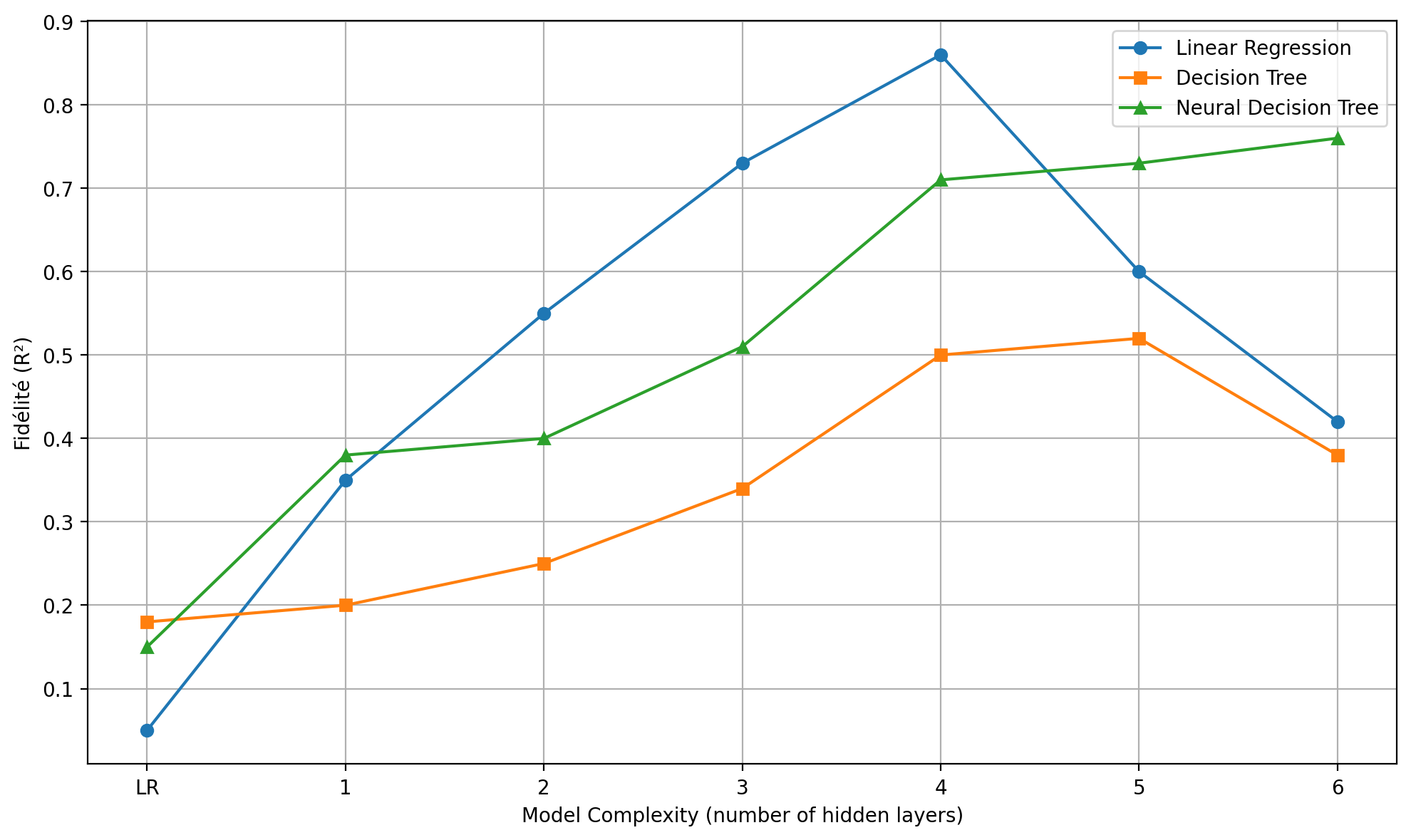}
			\caption{Diabetes}
			\label{subfig:Fidelity-diabetes}
		\end{subfigure}%
    \hfill
		\begin{subfigure}{\subfigsize\textwidth}
			\centering
      \includegraphics[width=\textwidth]{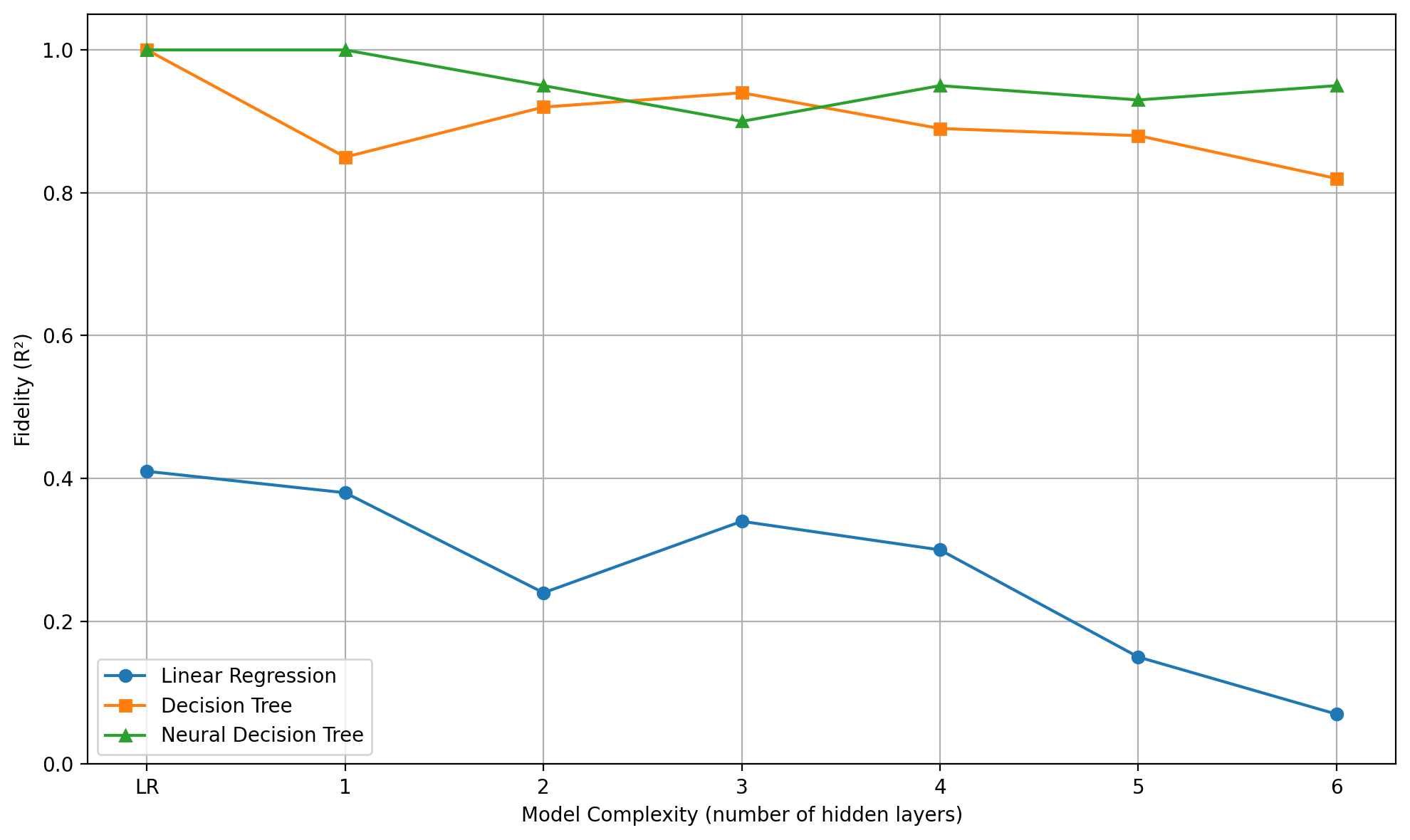}
			\caption{California Housing}
			\label{subfig:Fidelity-calif-housing}
		\end{subfigure}%
		\hfill
		\begin{subfigure}{\subfigsize\textwidth}
			\centering
      \includegraphics[width=\textwidth]{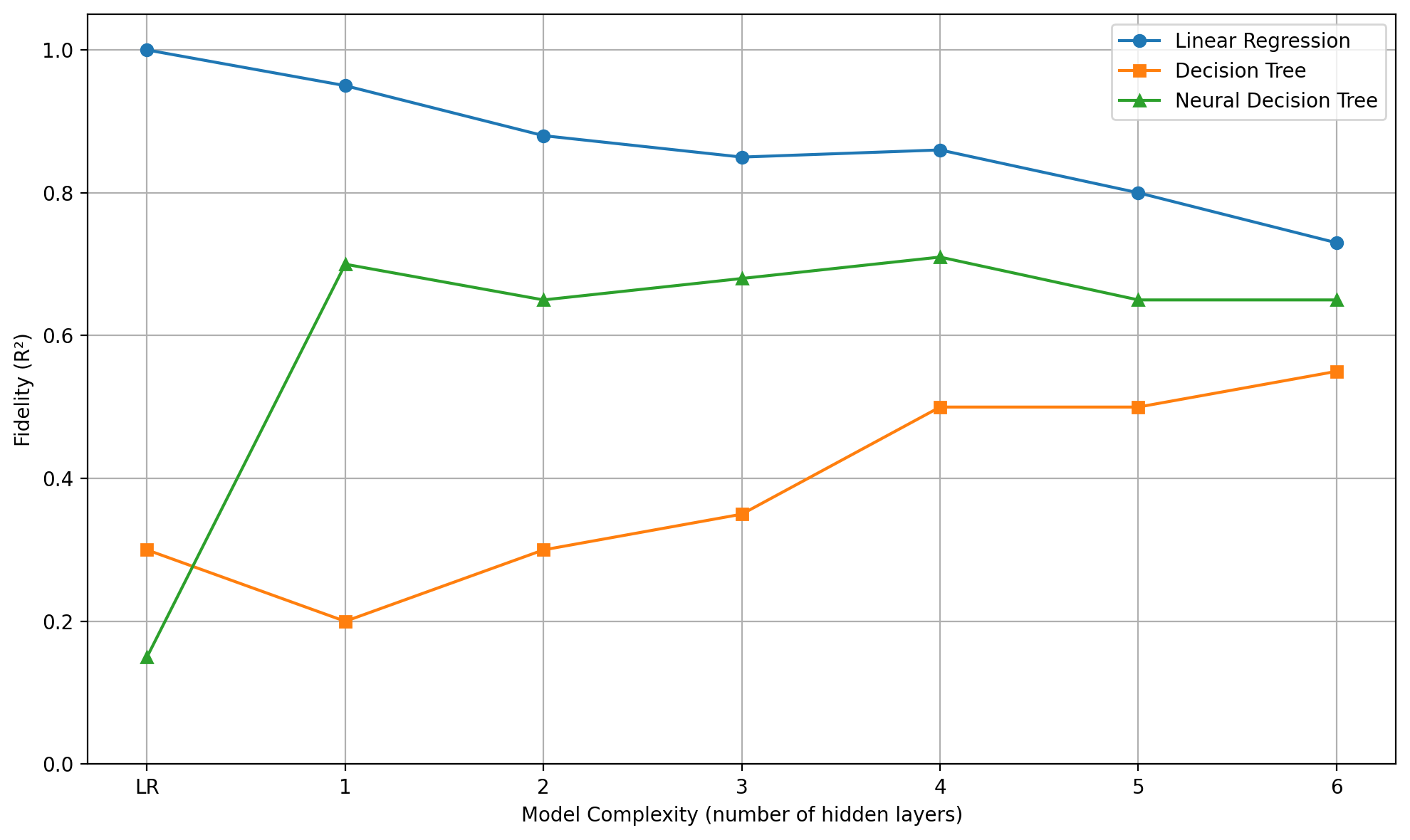}
			\caption{Ames Housing}
			\label{subfig:Fidelity-ames-housing}
		\end{subfigure}%
    \caption{Local fidelity comparison of LIME surrogates (\LIME{LR}, \LIME{DT}, and \NDTLIME) across increasing complexity of black-box models on benchmark datasets.}
    \label{fig:all_datasets}
\end{figure*}

\section{Analysis of Regularity in Different Local Regions}
\label{Analysis of Regularity in Different Local Regions}
The regularity score, defined by averaging the score of \Eq{eq:local_regularity} overall datapoints of a dataset, depends critically on the choice of $k$, the number of nearest neighbors used to assess the coherence of explanations across the input space. While small values of $k$ capture strictly local smoothness, larger values progressively include instances that may lie in different local regimes of the black-box decision function.

For inherently local explanation methods such as LIME, high regularity is expected to hold within a limited neighborhood. However, the way regularity decreases as $k$ increases reveals important differences between surrogate models, particularly their ability to handle decision boundaries and local variations.

To analyze this behavior, we vary $k$ from $1$ to $10$ and compute the global regularity score for \LIME{LR}, \LIME{DT}, and \NDTLIME across all benchmark datasets.

\begin{figure*}[t!]
	  \newcommand{\subfigsize}{0.24}
    \centering    
    \begin{subfigure}{\subfigsize\textwidth}
			\centering
      \includegraphics[width=\textwidth]{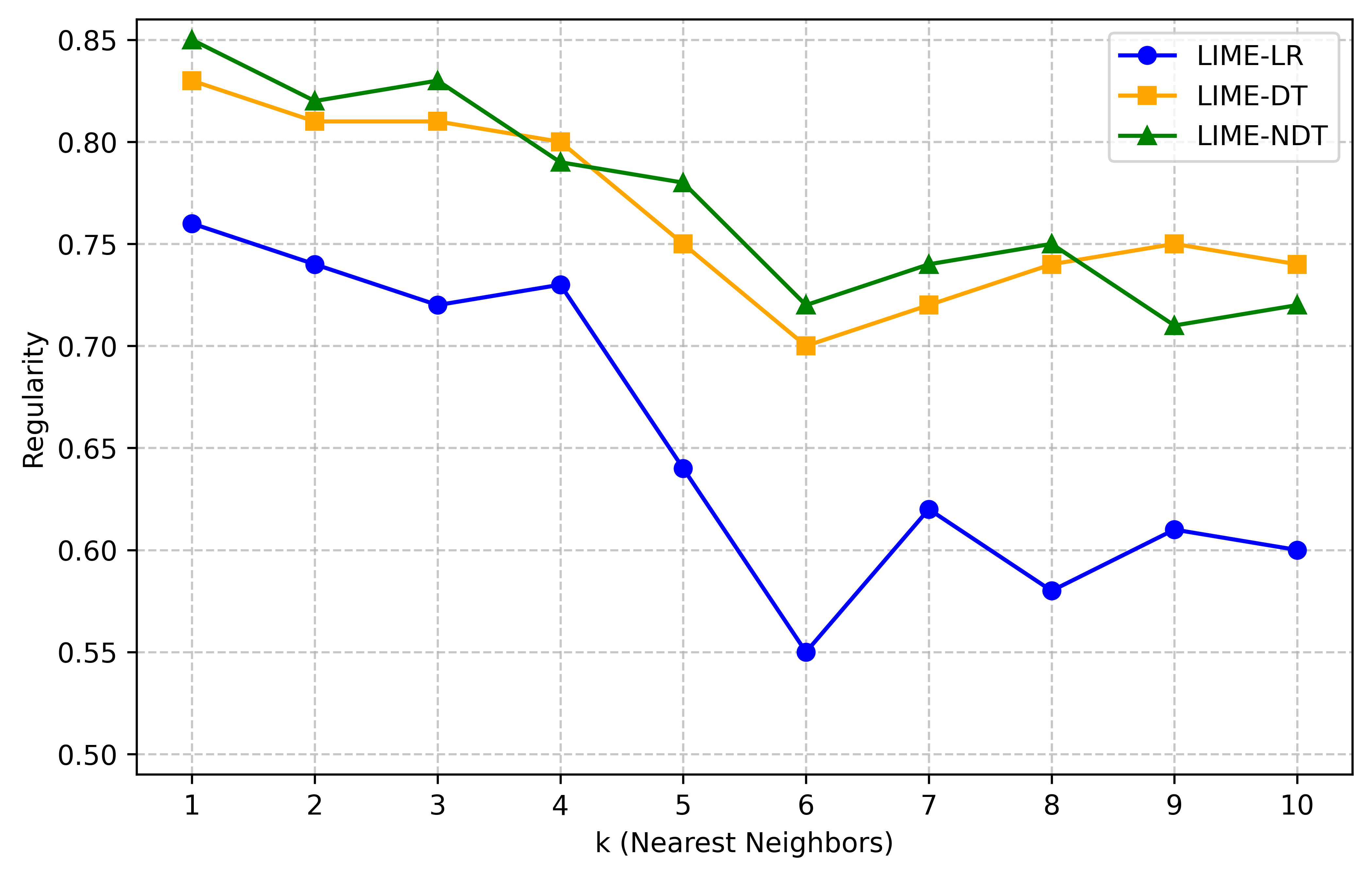}
			\caption{Iris}
			\label{subfig:Regularity-Iris}
		\end{subfigure}%
    \hfill%
    \begin{subfigure}{\subfigsize\textwidth}
			\centering
      \includegraphics[width=\textwidth]{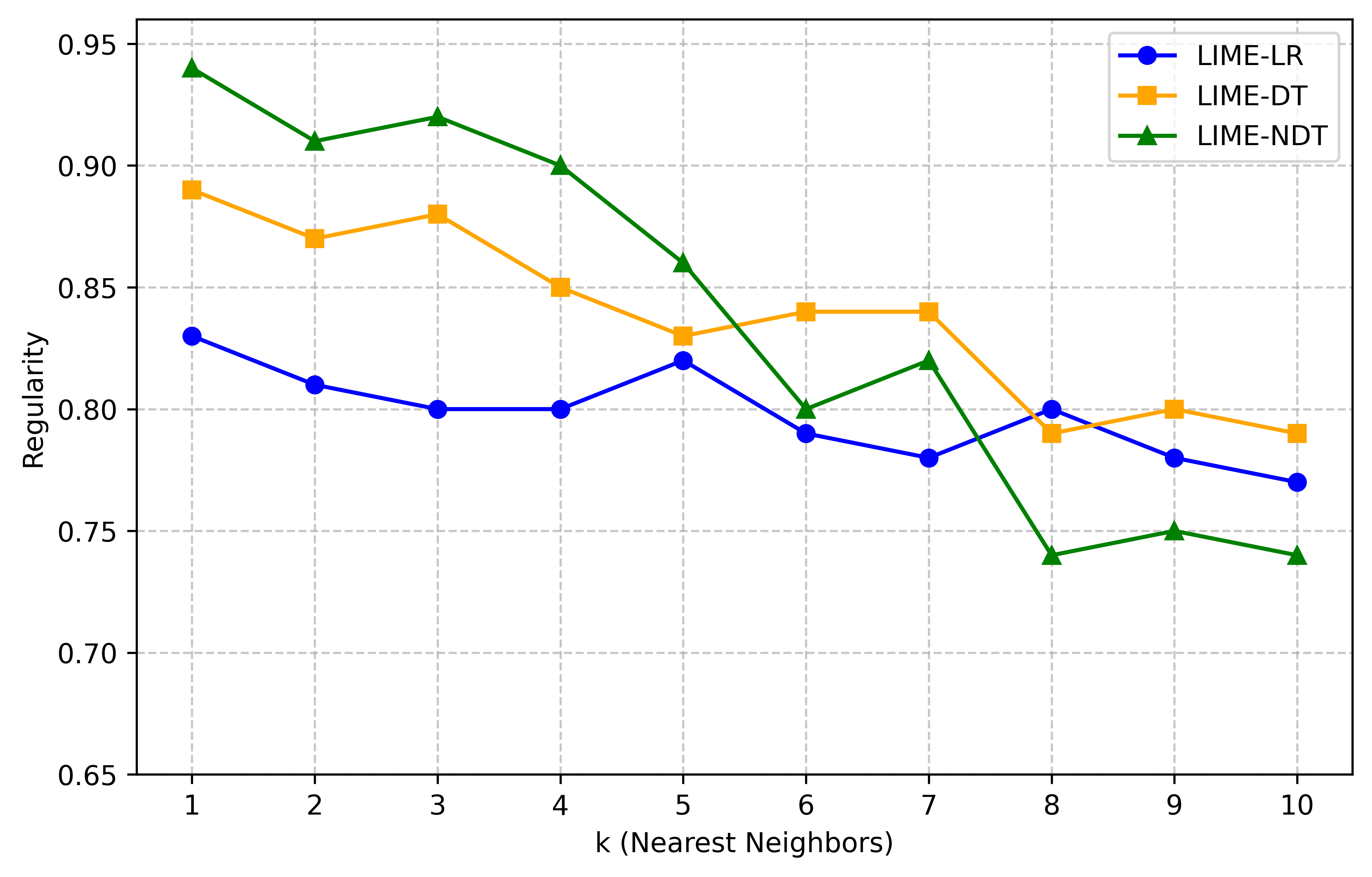}
			\caption{Breast Cancer}
			\label{subfig:Regularity-breast_cancer}
		\end{subfigure}%
		\hfill
    \begin{subfigure}{\subfigsize\textwidth}
			\centering
      \includegraphics[width=\textwidth]{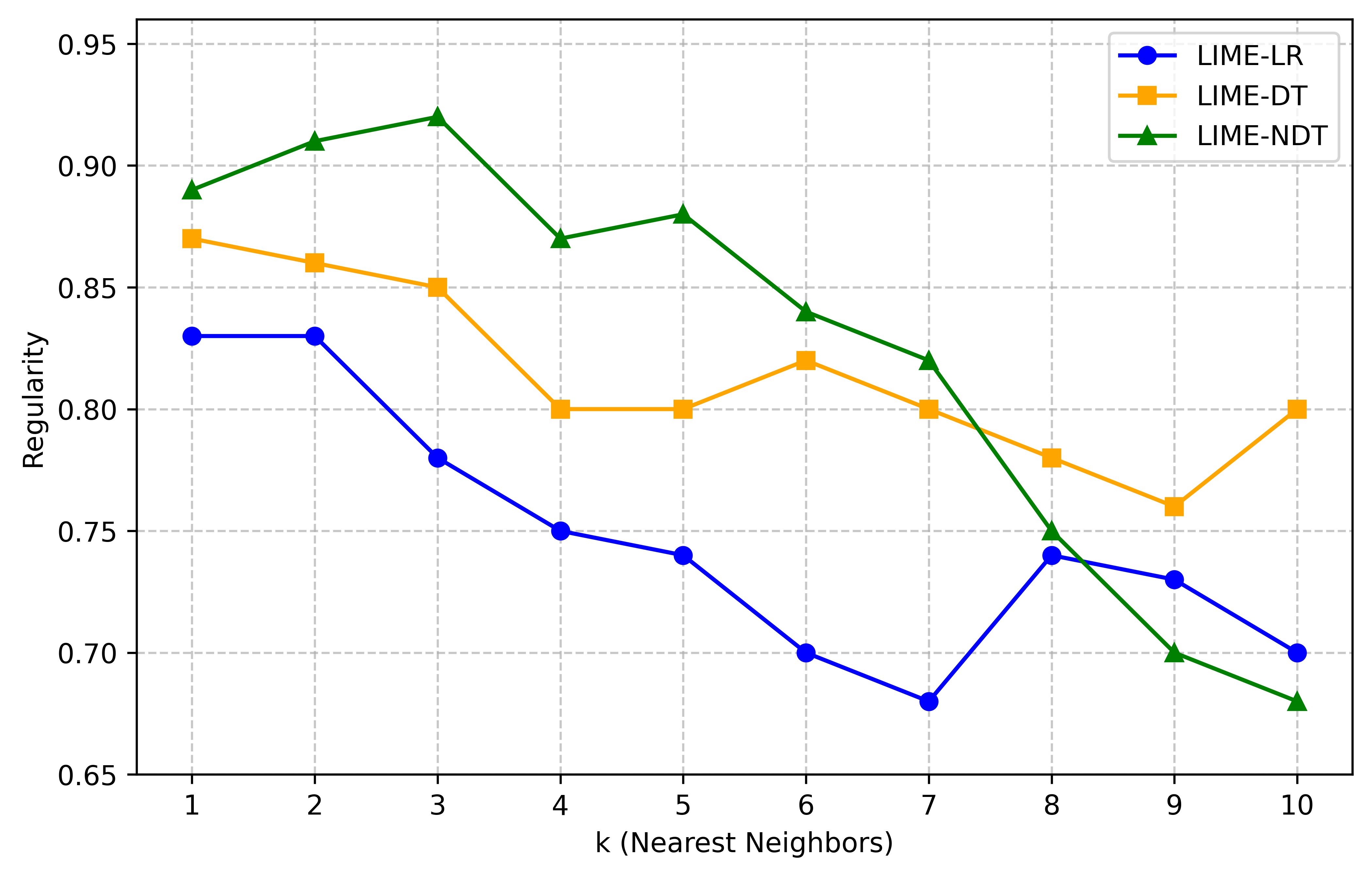}
			\caption{Wine}
			\label{subfig:Regularity-wine}
		\end{subfigure}%
			\hfill
    \begin{subfigure}{\subfigsize\textwidth}
			\centering
      \includegraphics[width=\textwidth]{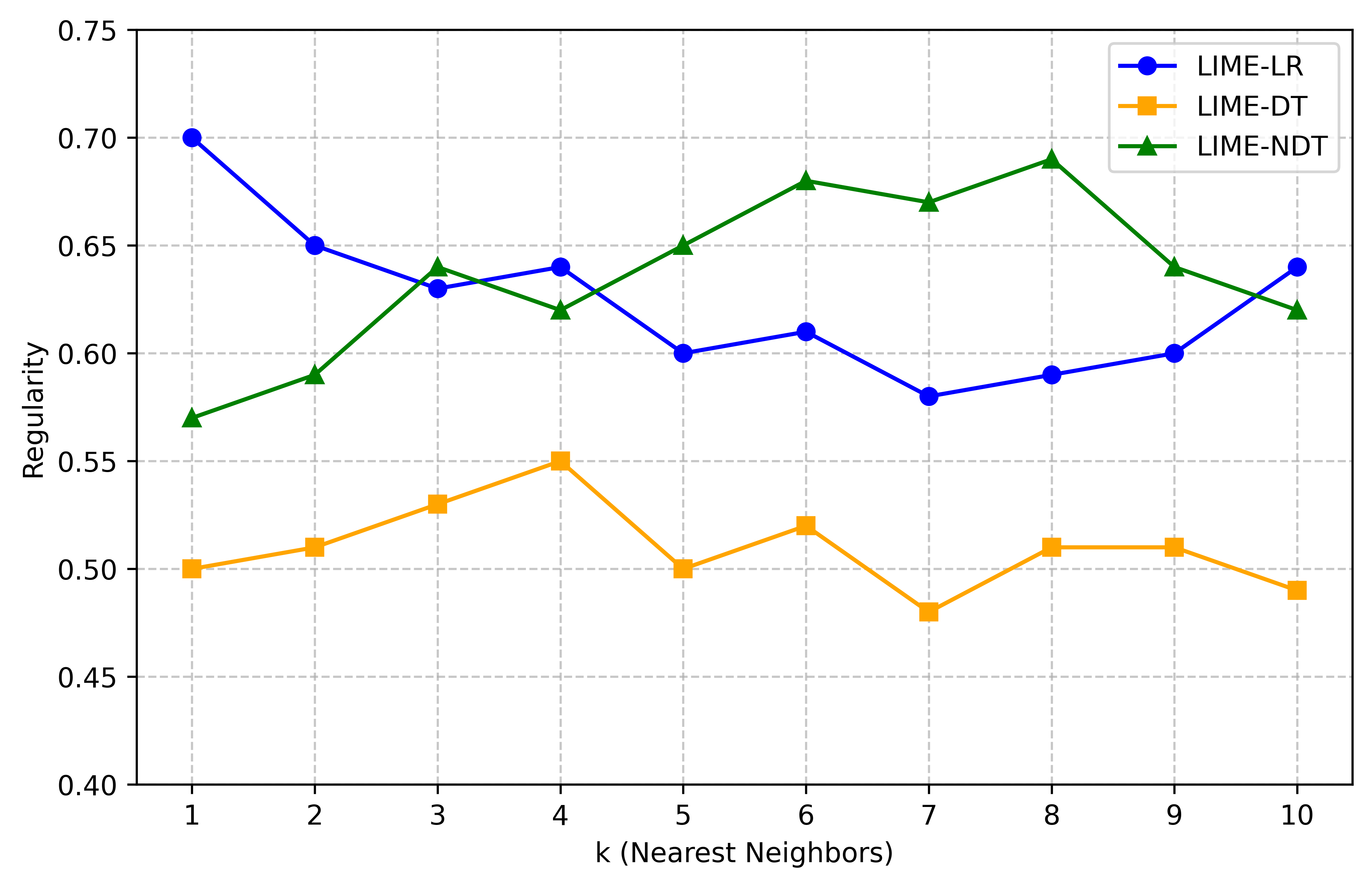}
			\caption{Digits}
			\label{subfig:Regularity-digits}
		\end{subfigure}%
		
    \vspace{1em} %
    
		\begin{subfigure}{\subfigsize\textwidth}
			\centering
      \includegraphics[width=\textwidth]{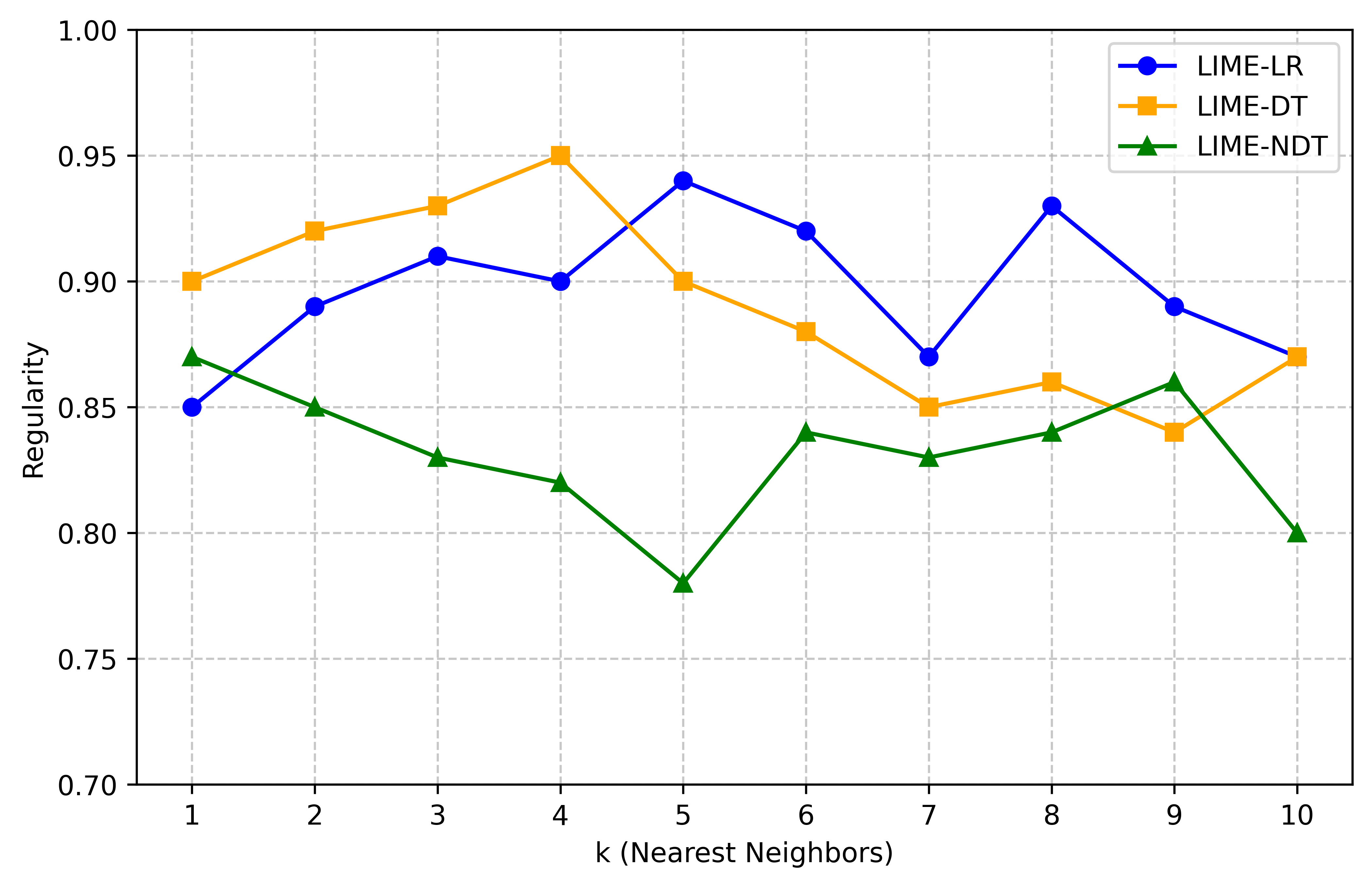}
			\caption{Covertype}
			\label{subfig:Regularity-covertype}
		\end{subfigure}%
		\hfill
		\begin{subfigure}{\subfigsize\textwidth}
			\centering
      \includegraphics[width=\textwidth]{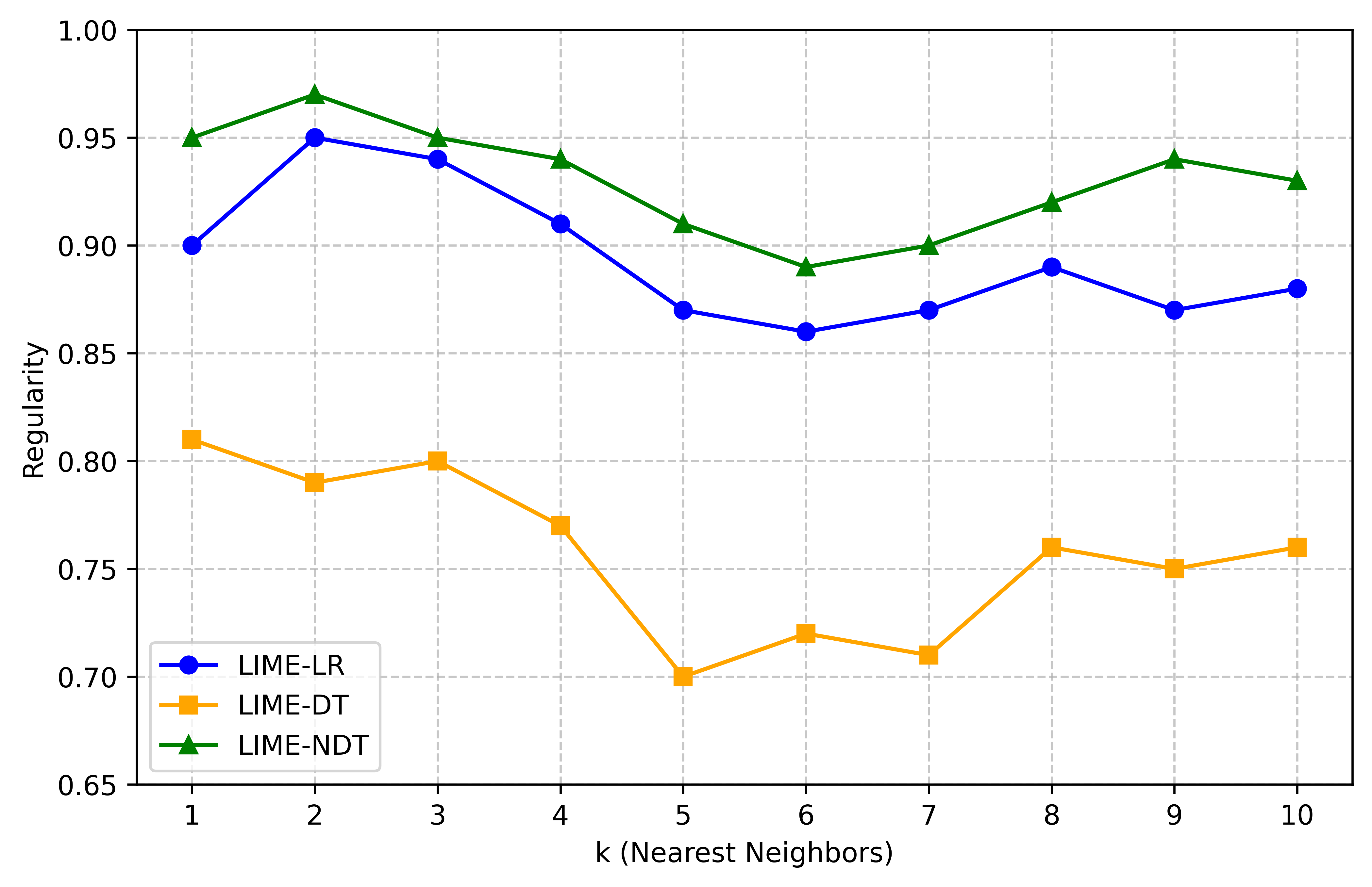}
			\caption{Diabetes}
			\label{subfig:Regularity-diabetes}
		\end{subfigure}%
    \hfill
		\begin{subfigure}{\subfigsize\textwidth}
			\centering
      \includegraphics[width=\textwidth]{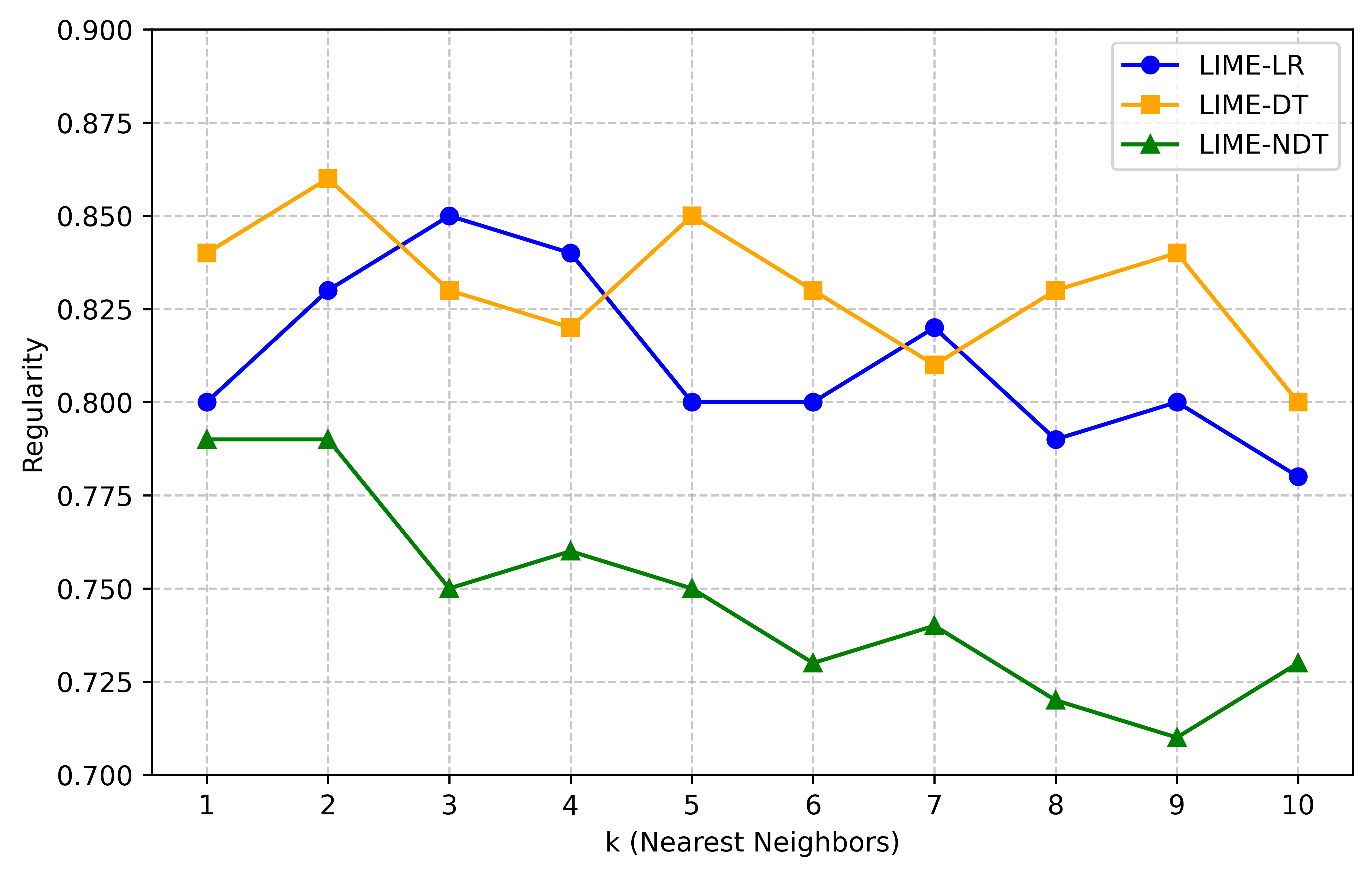}
			\caption{California Housing}
			\label{subfig:Regularity-calif-housing}
		\end{subfigure}%
		\hfill
		\begin{subfigure}{\subfigsize\textwidth}
			\centering
      \includegraphics[width=\textwidth]{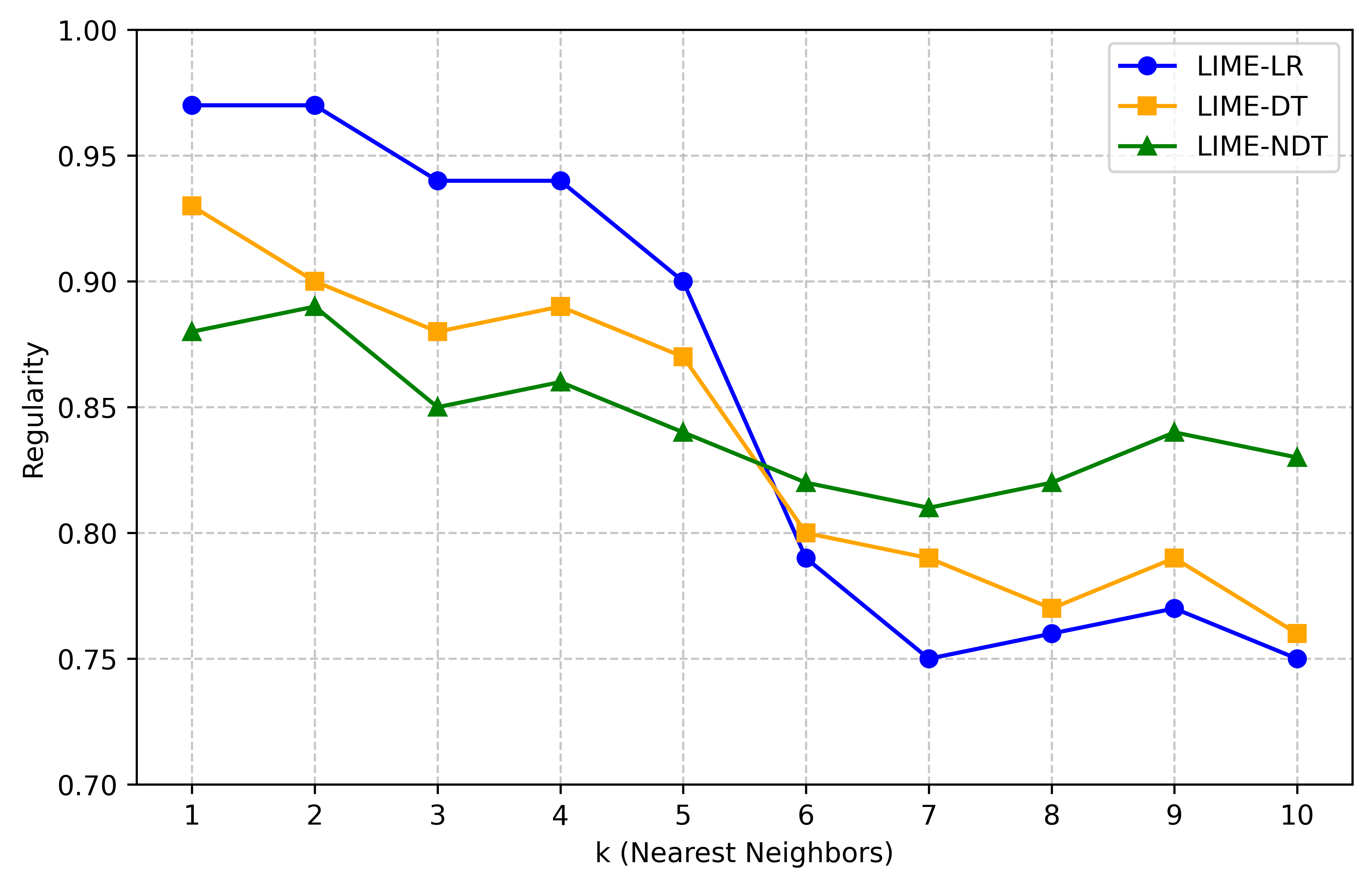}
			\caption{Ames Housing}
			\label{subfig:Regularity-ames-housing}
		\end{subfigure}%
    \caption{Regularity of LIME surrogates (\LIME{LR}, \LIME{DT}, and \NDTLIME) across increasing $k$-NNs on benchmark datasets.}
    \label{fig:regularity_all_datasets}
\end{figure*}

\LIME{LR} exhibits the slowest decay in regularity as $k$ increases. This behavior is a direct consequence of the global linear nature of the surrogate model, which produces explanation vectors that are nearly invariant across the input space. While this results in relatively high regularity values even for large k, it comes at the cost of poor local fidelity, as shown in \Sec{Measuring_Fidelity}. Consequently, the apparent robustness of \LIME{LR} at larger neighborhood scales reflects model rigidity rather than faithful local approximation of the black-box model.

\LIME{DT} shows a sharper decline in regularity as $k$ increases. The hard, axis-aligned splits induced by greedy decision tree training lead to abrupt changes in explanation vectors when instances cross split thresholds. As the neighborhood expands, neighboring instances are increasingly likely to fall into different leaf regions, resulting in sudden drops in cosine similarity between explanations. This behavior highlights the sensitivity of classical decision trees to neighborhood definitions and their limited ability to produce smooth explanation transitions.

In contrast, \NDTLIME demonstrates a more balanced and desirable behavior. For small values of $k$, \NDTLIME achieves high regularity, indicating strong local coherence between explanations of nearby instances. As $k$ increases, regularity decreases gradually rather than abruptly, reflecting the surrogate’s ability to adapt smoothly to evolving decision boundaries. This behavior is a direct consequence of the soft, differentiable splits of NDTs, which allow explanations to transition continuously across regions of the input space without introducing sharp discontinuities.

The rate at which regularity decreases with increasing $k$ varies across datasets and is influenced by intrinsic data characteristics such as dimensionality and class complexity. High-dimensional datasets such as Digits and Covtype exhibit faster regularity decay for all methods, reflecting the fragmented nature of their input spaces. In contrast, lower-dimensional datasets such as Iris and Wine preserve higher regularity over a wider range of $k$, indicating smoother underlying decision boundaries.

Overall, this analysis confirms that \NDTLIME produces explanations that are locally smooth yet globally adaptive. Unlike linear surrogates, which remain artificially consistent across large neighborhoods, and traditional decision trees, which suffer from abrupt explanation shifts, Neural Decision Trees strike a principled balance between continuity and expressiveness. This behavior aligns closely with the fundamental objective of local interpretability: providing explanations that are stable where they should be, and flexible where they must be.

\section{Description and implementation details for the main LIME competitors}\label{app:implementation-competitors}

\inlinetitle{GLIME}{}~%
 is a generalized framework that extends LIME to address its core limitations of instability and low local fidelity. These issues stem from LIME's sampling process, which is often non-local and biased, and from the dominance of the regularization term when sample weights are small. GLIME reformulates the explanation problem by integrating the weighting kernel directly into the sampling distribution.

\inlinetitle{DLIME}{}~%
is implemented following the original formulation, which does not rely on perturbation-based sampling. Instead of generating synthetic instances, DLIME operates directly on the original training data. An agglomerative hierarchical clustering is first applied to partition the dataset into locally homogeneous regions based on feature similarity. For a given instance to be explained, the corresponding cluster is identified using a nearest-neighbor assignment strategy. A local linear surrogate model is then trained exclusively on the datapoints belonging to this cluster to approximate the behavior of the black-box model in the vicinity of the instance. By defining locality through data-driven clustering rather than random perturbations, DLIME produces deterministic explanations and improves stability compared to standard LIME.

\inlinetitle{BayLIME}{}~%
is implemented by integrating Bayesian principles into the standard LIME workflow. It uses the same perturbation and weighting scheme as LR-LIME. The key difference lies in the surrogate model: instead of a standard linear regression, BayLIME employs a Bayesian Linear Regression model (BayesianRidge from scikit-learn). This model introduces a probabilistic prior over the regression coefficients, which helps regularize the model and makes it more robust to noise in the local data sample. The explanation is derived from the posterior mean of the surrogate's coefficients, representing the most probable feature importance values. We used the default non-informative, zero-centered prior for our experiments to ensure a fair and general comparison.

\inlinetitle{OptiLIME}{}~%
is implemented as a framework that optimizes the LIME hyperparameters, particularly the kernel width, to manage the trade-off between fidelity (adherence) and stability. Its implementation involves wrapping the standard LIME process within an optimization loop, typically using Bayesian optimization. For each instance, this optimizer systematically searches for the kernel width that yields the best possible explanation according to a combined objective function that rewards both high fidelity and high stability. This makes the explanation process more computationally intensive but tailored to each specific datapoint.

\inlinetitle{GMM-LIME}{}~%
modifies the local neighborhood definition. It first fits a Gaussian Mixture Model to the training data to identify underlying clusters. When explaining an instance, it uses the GMM to define a probabilistic local neighborhood based on the cluster the instance belongs to. Perturbed samples are then generated according to the learned Gaussian distribution of that specific cluster. This creates a more data-aware neighborhood than standard isotropic perturbations, aiming for improved local approximation and stability by aligning the sampling with the natural structure of the data.

\inlinetitle{sLIME}{}~%
brings as key innovation is an adaptive perturbation strategy. Its implementation requires modifying the initial data generation step. Instead of a fixed perturbation scale, sLIME introduces an algorithm to automatically determine the optimal neighborhood size for each instance. It iteratively adjusts the perturbation kernel width to find a balance where the local linear model is both stable and a good fit. The rest of the pipeline (weighting and training a linear surrogate) remains consistent with standard LIME.

\section{Details of the used datasets}\label{app:dataset-details}
This appendix provides a comprehensive description of the benchmark tabular datasets used to evaluate the proposed \NDTLIME framework. These datasets were selected to represent a diverse range of classification and regression problems, with varying feature dimensions, instance counts, and levels of non-linearity. All datasets were obtained from well-established machine learning repositories such as the UCI Repository and scikit-learn. A summary of these datasets, including their task type, size, and feature composition, is provided in \Tab{tab:dataset_details}. 

\begin{table*}[t]
\centering
\centerline{
\begin{tabular}{ll cc l}
\toprule
\textbf{Dataset} & \textbf{Task} & \textbf{\#\,Instances} & \textbf{\#\,Variables} & \textbf{Target variable}\\
\midrule
{Breast Cancer Wisconsin (Diagnostic)} %
& Classification & 569 & 30 & Diagnosis (malignant\,/\,benign) \\
{Iris} %
& Classification & 150 & 4 & Species of Iris flower ($3$ classes)\\
{Digits} %
& Classification & 1797 & 64 & Handwritten digit (0-9) \\
{Covertype (forest cover type) %
} 
& Classification & 581012 & 54 & Forest cover type (7 classes) \\
{Wine} %
& Classification & 178 & 13 & Cultivar (3 classes) \\
\midrule
{California Housing} %
& Regression & 20640 & 8 & Median house value \\
{Diabetes} %
& Regression & 442 & 10 & Disease progression score \\
{Ames Housing}%
& Regression 
& 2930 & 80 & Sale Price \\
\bottomrule
\end{tabular}
}
\caption{Benchmark tabular datasets used in the experiments.}
\label{tab:dataset_details} 
\end{table*}

\inlinetitle{Datasets used in other research papers}{.} 
We present different datasets used in works proposing other LIME variants in the related literature. This overview, summarized in \Tab{tab:dataset-list-literature}, contextualizes the selection of datasets in our study and highlights common benchmarks used to evaluate local explanation methods.

Several variants, including GLIME, MPS-LIME, and KL-LIME, were primarily evaluated on non-tabular data, such as images or text, and therefore did not use the tabular datasets common in this paper. For tabular data, DLIME experiments were conducted on the Breast Cancer Wisconsin, Indian Liver Patient, and Hepatitis datasets, focusing on medical diagnosis scenarios. The MeLIME variant used the classic Iris dataset to demonstrate its approach, while SLIME was validated on the Breast Cancer Wisconsin, MARS test function, and Sepsis prediction datasets. The iLIME variant was tested on a broader range of classification tasks, including the Iris, Wine, Breast Cancer, Cervical Cancer, Diabetes, Hypertension, and Mortality datasets. Similarly, BayLIME used the Boston House-Price and Breast Cancer Wisconsin datasets to showcase its Bayesian framework. Other methods, like OptiLIME, used synthetic toy data and the NHANES I dataset, and GMM-LIME evaluations were performed on the Human Gait Database (HugaDB) and GaitRec for interpreting sensor-based models.

This compilation shows a diverse but somewhat fragmented use of datasets across different LIME variants, underscoring the need for standardized, comprehensive benchmarks like those employed in our work to enable fair and robust comparisons.

\begin{table}\small
\begin{tabular}{ll}
\toprule
LIME variant & Datasets used \\
\midrule
GLIME & No tabular data \\
DLIME & Breast Cancer Wisconsin, Indian Liver Patient,\\ 
& Hepatitis\\
MPS-LIME & No tabular data\\
KL-LIME & No tabular data\\
MeLIME & Iris\\
sLIME & Breast Cancer Wisconsin, MARS test function, \\ 
& Sepsis prediction\\
iLIME & Iris, Wine, Breast Cancer, Cervical Cancer, \\
& Diabetes, Hypertension, Mortality\\
BayLIME & Boston House-Price, Breast Cancer Wisconsin\\
OptiLIME & Toy data, NHANES I\\
GMM-LIME & Human Gait Database (HugaDB), GaitRec\\
\bottomrule
\end{tabular}%
\caption{A list of the datasets used in the LIME-related literature.}\label{tab:dataset-list-literature}
\end{table}

\begin{table*}[t]
 \small
 \centering
\centerline{
 \resizebox{1.4\textwidth}{!}{
 \begin{tabular}{lcccccccc}
 \toprule
 \textbf{LIME Variant}   & \textbf{Breast Cancer}      & \textbf{Iris}         & \textbf{Wine}         & \textbf{Digits}       & \textbf{Covtype}      & \textbf{CA Housing} & \textbf{Diabetes}      & \textbf{Ames Housing}   \\
 \midrule
  GLIME      & \textbf{0.825 $\pm$ 0.012} & 0.640 $\pm$ 0.021 & 0.420 $\pm$ 0.015 & \textit{0.510 $\pm$ 0.011} & \textit{0.590 $\pm$ 0.025} & 0.465 $\pm$ 0.031 & \textit{0.907 $\pm$ 0.028} & \textbf{0.913 $\pm$ 0.022} \\
  DLIME      & 0.435 $\pm$ 0.009 & 0.563 $\pm$ 0.025 & 0.395 $\pm$ 0.018 & 0.285 $\pm$ 0.013 & 0.410 $\pm$ 0.028 & 0.260 $\pm$ 0.035 & 0.760 $\pm$ 0.031 & 0.751 $\pm$ 0.025 \\
  MPS-LIME   & 0.612 $\pm$ 0.015 & 0.600 $\pm$ 0.028 & 0.365 $\pm$ 0.021 & 0.260 $\pm$ 0.016 & 0.385 $\pm$ 0.031 & 0.342 $\pm$ 0.040 & 0.840 $\pm$ 0.035 & 0.830 $\pm$ 0.029 \\
  KL-LIME    & 0.514 $\pm$ 0.016 & 0.585 $\pm$ 0.031 & 0.350 $\pm$ 0.023 & 0.245 $\pm$ 0.018 & 0.370 $\pm$ 0.033 & 0.321 $\pm$ 0.045 & 0.798 $\pm$ 0.038 & \textit{0.871 $\pm$ 0.032} \\
  MeLIME    & 0.547 $\pm$ 0.018 & 0.630 $\pm$ 0.033 & 0.435 $\pm$ 0.025 & 0.514 $\pm$ 0.020 & 0.355 $\pm$ 0.035 & 0.415 $\pm$ 0.048 & 0.812 $\pm$ 0.040 & 0.808 $\pm$ 0.035 \\
  sLIME      & 0.718 $\pm$ 0.005 & 0.610 $\pm$ 0.026 & 0.405 $\pm$ 0.017 & 0.457 $\pm$ 0.012 & 0.468 $\pm$ 0.027 & \textit{0.502 $\pm$ 0.033} & 0.853 $\pm$ 0.026 & 0.834 $\pm$ 0.023 \\
  iLIME      & 0.688 $\pm$ 0.014 & \textit{0.655 $\pm$ 0.019} & \textit{0.440 $\pm$ 0.013} & 0.325 $\pm$ 0.010 & 0.460 $\pm$ 0.023 & 0.369 $\pm$ 0.028 & 0.882 $\pm$ 0.024 & 0.852 $\pm$ 0.020 \\
  BayLIME    & 0.739 $\pm$ 0.022 & 0.624 $\pm$ 0.030 & 0.418 $\pm$ 0.024 & 0.476 $\pm$ 0.019 & 0.415 $\pm$ 0.034 & 0.397 $\pm$ 0.043 & 0.894 $\pm$ 0.037 & 0.845 $\pm$ 0.030 \\
  OptiLIME  & 0.559 $\pm$ 0.013 & 0.570 $\pm$ 0.022 & 0.415 $\pm$ 0.016 & 0.305 $\pm$ 0.011 & 0.435 $\pm$ 0.026 & 0.298 $\pm$ 0.030 & 0.817 $\pm$ 0.025 & 0.720 $\pm$ 0.021 \\
  GMM-LIME   & 0.621 $\pm$ 0.025 & 0.545 $\pm$ 0.038 & 0.320 $\pm$ 0.028 & 0.215 $\pm$ 0.022 & 0.335 $\pm$ 0.038 & 0.295 $\pm$ 0.051 & 0.775 $\pm$ 0.043 & 0.798 $\pm$ 0.038 \\
	\textbf{NDT (ours)} & \textit{0.785 $\pm$ 0.031} & \textbf{0.860 $\pm$ 0.021} & \textbf{0.518 $\pm$ 0.131} & \textbf{0.577 $\pm$ 0.106} & \textbf{0.632 $\pm$ 0.067} & \textbf{0.960 $\pm$ 0.014} & \textbf{0.920 $\pm$ 0.035} & 0.713
    $\pm$ 0.054\\
 \bottomrule
 \end{tabular}
 }}
 \caption{Fidelity of different LIME variants on benchmark datasets.}
\label{tab:fidelity_lime_variants}
\end{table*}

\begin{table*}[t]
\small
\centering
\centerline{
\resizebox{1.4\textwidth}{!}{
\begin{tabular}{lcccccccc}
\toprule
\textbf{LIME Variant}   & \textbf{Breast Cancer}      & \textbf{Iris}         & \textbf{Wine}         & \textbf{Digits}       & \textbf{Covtype}      & \textbf{CA Housing} & \textbf{Diabetes}      & \textbf{Ames Housing}   \\
\midrule
GLIME          & \textit{0.996 $\pm$ 0.003}  & 0.987 $\pm$ 0.011  & 0.997 $\pm$ 0.001 & \textit{0.990 $\pm$ 0.004}  & 0.991 $\pm$ 0.010  & \textit{0.999 $\pm$ 0.000}   & \textit{0.999 $\pm$ 0.003}  & 0.994 $\pm$ 0.007   \\
DLIME          & \textbf{1.000 $\pm$ 0.000}  & \textbf{1.000 $\pm$ 0.000} & \textbf{1.000 $\pm$ 0.000} & \textbf{1.000 $\pm$ 0.000} & \textbf{1.000 $\pm$ 0.000} & \textbf{1.000 $\pm$ 0.000}   & \textbf{1.000 $\pm$ 0.000}  & \textbf{1.000 $\pm$ 0.000}  \\
MPS-LIME       & 0.956 $\pm$ 0.005  & 0.963 $\pm$ 0.008  & 0.974 $\pm$ 0.003  & 0.908 $\pm$ 0.018  & 0.954 $\pm$ 0.024  & 0.973 $\pm$ 0.009    & 0.972 $\pm$ 0.004  & 0.965 $\pm$ 0.010   \\
KL-LIME        & 0.960 $\pm$ 0.007  & 0.954 $\pm$ 0.013  & 0.965 $\pm$ 0.010  & 0.938 $\pm$ 0.024  & 0.949 $\pm$ 0.018  & 0.960 $\pm$ 0.007    & 0.971 $\pm$ 0.002  & 0.954 $\pm$ 0.014   \\
MeLIME         & 0.949 $\pm$ 0.010  & \textit{0.995 $\pm$ 0.004}  & 0.934 $\pm$ 0.012  & 0.973 $\pm$ 0.012  & 0.982 $\pm$ 0.005  & 0.963 $\pm$ 0.003    & \textit{0.999 $\pm$ 0.000}  & 0.995 $\pm$ 0.004   \\
sLIME          & 0.995 $\pm$ 0.005  & 0.986 $\pm$ 0.003 & \textit{0.999 $\pm$ 0.000} & 0.985 $\pm$ 0.005 & \textit{0.999 $\pm$ 0.002} & \textit{0.999 $\pm$ 0.001}   & \textit{0.999 $\pm$ 0.001}  & \textit{0.998 $\pm$ 0.002}  \\
iLIME          & 0.975 $\pm$ 0.004  & 0.974 $\pm$ 0.010  & 0.984 $\pm$ 0.005 & 0.953 $\pm$ 0.022  & 0.972 $\pm$ 0.015  & 0.976 $\pm$ 0.005   & 0.978 $\pm$ 0.005  & 0.973 $\pm$ 0.011   \\
BayLIME        & 0.940 $\pm$ 0.015  & 0.935 $\pm$ 0.043  & 0.948 $\pm$ 0.027  & 0.910 $\pm$ 0.055  & 0.926 $\pm$ 0.025  & 0.944 $\pm$ 0.013    & 0.951 $\pm$ 0.009  & 0.939 $\pm$ 0.036   \\
OptiLIME      & 0.980 $\pm$ 0.003  & 0.974 $\pm$ 0.017  & 0.982 $\pm$ 0.005 & 0.968 $\pm$ 0.017  & 0.975 $\pm$ 0.015  & 0.987 $\pm$ 0.002   & 0.984 $\pm$ 0.005  & 0.993 $\pm$ 0.002   \\
GMM-LIME       & 0.930 $\pm$ 0.020  & 0.926 $\pm$ 0.006  & 0.936 $\pm$ 0.024  & 0.908 $\pm$ 0.030  & 0.917 $\pm$ 0.021  & 0.934 $\pm$ 0.010    & 0.947 $\pm$ 0.008  & 0.925 $\pm$ 0.036   \\
	\textbf{NDT (ours)} & 0.991 $\pm$ 0.004 & 0.943 $\pm$ 0.010 & 0.998 $\pm$ 0.002 & 0.816 $\pm$ 0.023 & 0.931 $\pm$ 0.007 & 0.973 $\pm$ 0.001 & 0.998 $\pm$ 0.002 & 0.990 $\pm$ 0.013\\
\bottomrule
\end{tabular}
}}
\caption{Stability of different LIME variants on benchmark datasets.}
\label{tab:stability_lime_variants}
\end{table*}

\begin{table*}[t]
\small
\centering
\centerline{
\resizebox{1.4\textwidth}{!}{
\begin{tabular}{lcccccccc}
\toprule
\textbf{LIME Variant}   & \textbf{Breast Cancer}      & \textbf{Iris}         & \textbf{Wine}         & \textbf{Digits}       & \textbf{Covtype}      & \textbf{CA Housing} & \textbf{Diabetes}      & \textbf{Ames Housing}   \\
\midrule
GLIME          & 0.905 $\pm$ 0.014  & 0.812 $\pm$ 0.021 & 0.842 $\pm$ 0.018 & 0.597 $\pm$ 0.027  & \textit{0.894 $\pm$ 0.019}  & \textbf{0.869 $\pm$ 0.016}   & \textit{0.952 $\pm$ 0.012}  & \textbf{0.962 $\pm$ 0.015}  \\
DLIME          & 0.718 $\pm$ 0.012  & 0.725 $\pm$ 0.019 & 0.675 $\pm$ 0.017 & 0.508 $\pm$ 0.025  & 0.806 $\pm$ 0.017  & 0.756 $\pm$ 0.015   & 0.844 $\pm$ 0.011  & 0.814 $\pm$ 0.014  \\
MPS-LIME       & 0.782 $\pm$ 0.018  & 0.761 $\pm$ 0.026 & 0.764 $\pm$ 0.022 & 0.622 $\pm$ 0.031  & 0.811 $\pm$ 0.024  & 0.812 $\pm$ 0.021   & 0.905 $\pm$ 0.017  & 0.853 $\pm$ 0.020  \\
KL-LIME        & 0.803 $\pm$ 0.020  & 0.722 $\pm$ 0.028 & 0.739 $\pm$ 0.025 & 0.623 $\pm$ 0.034  & 0.816 $\pm$ 0.026  & 0.755 $\pm$ 0.023   & 0.828 $\pm$ 0.019  & 0.818 $\pm$ 0.022  \\
MeLIME         & 0.852 $\pm$ 0.021  & 0.748 $\pm$ 0.030 & 0.802 $\pm$ 0.027 & 0.610 $\pm$ 0.036  & \textbf{0.901 $\pm$ 0.028}  & 0.781 $\pm$ 0.025   & 0.874 $\pm$ 0.020  & 0.844 $\pm$ 0.024  \\
sLIME          & 0.888 $\pm$ 0.010  & 0.746 $\pm$ 0.016 & \textit{0.892 $\pm$ 0.013} & \textbf{0.731 $\pm$ 0.021}  & 0.881 $\pm$ 0.015  & 0.819 $\pm$ 0.013   & 0.958 $\pm$ 0.008  & 0.826 $\pm$ 0.011  \\
iLIME          & \textit{0.912 $\pm$ 0.013}  & 0.831 $\pm$ 0.018 & 0.879 $\pm$ 0.015 & 0.714 $\pm$ 0.023  & 0.889 $\pm$ 0.017  & \textit{0.858 $\pm$ 0.014}   & 0.946 $\pm$ 0.010  & \textit{0.918 $\pm$ 0.013}  \\
BayLIME        & 0.843 $\pm$ 0.024  & 0.739 $\pm$ 0.034 & 0.795 $\pm$ 0.029 & 0.602 $\pm$ 0.041  & 0.814 $\pm$ 0.031  & 0.773 $\pm$ 0.027   & 0.868 $\pm$ 0.022  & 0.839 $\pm$ 0.026  \\
OptiLIME      & 0.911 $\pm$ 0.011  & \textit{0.838 $\pm$ 0.019} & 0.886 $\pm$ 0.014 & \textit{0.724 $\pm$ 0.022}  & 0.866 $\pm$ 0.016  & 0.864 $\pm$ 0.013   & 0.912 $\pm$ 0.009  & 0.902 $\pm$ 0.012  \\
GMM-LIME       & 0.817 $\pm$ 0.027  & 0.721 $\pm$ 0.039 & 0.781 $\pm$ 0.033 & 0.538 $\pm$ 0.044  & 0.772 $\pm$ 0.035  & 0.762 $\pm$ 0.030   & 0.857 $\pm$ 0.025  & 0.826 $\pm$ 0.029  \\
	\textbf{NDT (ours)} & \textbf{0.915 $\pm$ 0.024} & \textbf{0.820 $\pm$ 0.039} & \textbf{0.910 $\pm$ 0.023} & 0.563 $\pm$ 0.034 & 0.849 $\pm$ 0.072 & 0.794 $\pm$ 0.033 & \textbf{0.978 $\pm$ 0.017} & 0.893 $\pm$ 0.027\\
\bottomrule
\end{tabular}
}}
\caption{Regularity of different LIME variants on benchmark datasets.}
\label{tab:regularity_lime_variants}
\end{table*}

\section{Smoothness and local Taylor expansion of Neural Decision Trees}\label{app:continuous-NDT}

In this appendix, we formally establish the smoothness properties of Neural Decision Trees (NDTs) and discuss their implications for local interpretability. Our goal is to demonstrate that NDTs belong to the class of $C^\infty(\mathbb{R}^d)$ functions, which allows for the use of Taylor expansions to characterize their local behavior. This smoothness is a critical advantage over traditional decision trees, as it enables continuous and differentiable explanations that better capture local decision boundaries. To achieve this, we provide proofs of the smoothness of NDTs and derive their local Taylor expansion, thereby highlighting how this property contributes to the interpretability and stability of the explanations generated by NDT-LIME.

\subsection{Neural Decision Trees are $C^\infty$}
\label{NDT_C_infinite}

\inlinetitle{Definition}{.}~%
A Neural Decision Tree (NDT) defines a function
$g : \mathbb{R}^d \to \mathbb{R}$ of the form:
\begin{equation}
g(x) = \sum_{k=1}^K  a_k  \prod_{m \in \mathcal P_k} \sigma_\gamma(x_{j_m} - \alpha_m),
\label{eq:ndt_def}
\end{equation}
where the sum goes over each $k=1,...,K$ root-to-leaf path, $\mathcal P_k$ is the set of splits along that path, $a_k \in \mathbb{R}$ are leaf values, and $\sigma_\gamma$ is a smooth gating function (\eg $\sigma_\gamma(u)=\tanh(\gamma u)$).
\begin{lemma}
The function $\sigma_\gamma(u)=\tanh(\gamma u)$ belongs to $C^\infty(\mathbb{R})$.
\end{lemma}
\begin{proof}
The hyperbolic tangent function is analytic on $\mathbb{R}$, and composition with an affine map preserves smoothness.
\end{proof}
\begin{lemma}
Finite sums and finite products of $C^\infty$ functions are $C^\infty$.
\end{lemma}
\begin{proof}
This follows directly from standard differentiation rules.
\end{proof}
\begin{proposition}
The function $g$ defined in~\Eq{eq:ndt_def} belongs to $C^\infty(\mathbb{R}^d)$.
\end{proposition}
\begin{proof}
Each gating function $\sigma_\gamma(x_{j_m}-\alpha_m)$ is $C^\infty$ by the previous lemma.  
Each path contribution is a finite product of $C^\infty$ functions and is therefore $C^\infty$.  
Since $g$ is a finite sum of such terms, we conclude that $g \in C^\infty(\mathbb{R}^d)$.
\end{proof}

In particular, NDTs define functions in $C^2(\mathbb{R}^d)$, which is sufficient for second-order Taylor expansions.~\ref{Taylor_approximation}

\subsection{Local Taylor Expansion of an NDT}
\label{Taylor_approximation}
Let $x_0 \in \mathbb{R}^d$ be a point of interest.  
Since $g \in C^2(\mathbb{R}^d)$, Taylor's theorem applies in a neighborhood of $x_0$.

\begin{proposition}[Second-order Taylor expansion]
For any $h \in \mathbb{R}^d$ sufficiently small,
\begin{equation}
g(x_0 + h) 
=
g(x_0)
+
\nabla g(x_0)^\top h
+
\frac{1}{2} h^\top H_g(x_0) h
+
o(\|h\|^2),
\label{eq:taylor_ndt}
\end{equation}
where $\nabla g(x_0)$ denotes the gradient and $H_g(x_0)$ the Hessian matrix of $g$ at $x_0$.
\end{proposition}

The first-order term $\nabla g(x_0)$ provides well-defined and continuous feature attributions, while the second-order term captures local curvature and interactions between features that appear jointly along decision paths. This contrasts with linear surrogates, whose Hessian is identically zero, and with hard decision trees, which do not admit a Taylor expansion due to discontinuities.

\section{Theoretical limits of local fidelity}

\subsection{Local fidelity as a constrained approximation problem}

Let $f : \mathbb{R}^d \to \mathbb{R}$ be a black-box model and let $x_0 \in \mathbb{R}^d$ be a point of interest.
We consider a locality distribution $\pi_{x_0}$ supported on a neighborhood
$U = B(x_0,r)$.

For a given class of surrogate models $\mathcal G$, we define the minimal achievable
local fidelity error as
\begin{equation}
\mathcal E_{\mathcal G}(f;x_0)
=
\inf_{g \in \mathcal{G}}
\| f - g \|_{L^2(\pi_{x_0})}.
\label{eq:fidelity_def}
\end{equation}
This quantity characterizes intrinsic limitations of the surrogate class,
independently of optimization or sampling effects.

\subsection{Linear surrogates (LR-LIME)}
Assume that $f$ is twice continuously differentiable in a neighborhood of $x_0$ and that
\[
H_f(x_0) \neq 0,
\]
where $H_f(x_0)$ denotes the Hessian of $f$ at $x_0$.
Let $\mathcal G_{\mathrm{LR}}$ denote the class of affine functions.

\begin{proposition}[Lower bound for linear surrogates, informal]
For sufficiently small $r$, there exists a constant $c>0$ such that
\begin{equation}
\mathcal E_{\mathcal{G}_{\mathrm{LR}}}(f;x_0)
\;\ge\;
c\, r^2 \, \|H_f(x_0)\|.
\label{eq:lr_lower_bound}
\end{equation}
\end{proposition}
Linear surrogates can match the first-order Taylor expansion of $f$, but cannot represent local curvature. As a consequence, the approximation error remains strictly positive, even in an arbitrarily small neighborhood.

\subsection{Decision tree surrogates (DT-LIME)}
Let $\mathcal{G}_{\mathrm{DT},d}$ denote the class of decision trees of depth $d$,
which define piecewise-constant functions on partitions of size $\mathcal O(2^d)$.
Assume that $f$ is locally Lipschitz in $U$.

\begin{proposition}[Approximation rate of decision trees, informal]
There exists a constant $C>0$ such that
\begin{equation}
\mathcal E_{\mathcal{G}_{\mathrm{DT},d}}(f;x_0)
=
\mathcal O\!\left(2^{-d/d_{\mathrm{eff}}}\right),
\label{eq:dt_rate}
\end{equation}
where $d_{\mathrm{eff}}$ denotes the effective local dimension of the input.
\end{proposition}

Increasing the tree depth refines the local partition, but approximation is limited by discontinuities at split boundaries and lack of smoothness within each cell, leading to slow convergence with depth.

\subsection{Neural Decision Trees (NDT-LIME)}

Let $\mathcal{G}_{\mathrm{NDT},d}$ denote the class of NDTs of depth $d$ with smooth gating functions. Assume that $f \in C^2(U)$. NDTs define smooth, non-linear surrogates that can locally approximate both gradient and curvature information.

\begin{proposition}[Local approximation by NDTs, informal]
There exists a constant $C>0$ such that
\begin{equation}
\mathcal E_{\mathcal{G}_{\mathrm{NDT},d}}(f;x_0)
=
\mathcal O\!\left(2^{-2d/d_{\mathrm{eff}}}\right).
\label{eq:ndt_rate}
\end{equation}
\end{proposition}
NDTs behave as piecewise-smooth approximators, analogous to spline models. They combine the expressive power of non-linear models with the regularity of smooth functions, yielding an intermediate convergence regime.

\subsection{Comparison and takeaway}
Combining the above results yields the hierarchy
\begin{equation}
\mathcal{E}_{\mathrm{LR}}
\;\gg\;
\mathcal{E}_{\mathrm{NDT},d}
\;\gg\;
\mathcal{E}_{\mathrm{DT},d},
\end{equation}
highlighting the structural advantage of Neural Decision Trees as local surrogates.

The bounds presented above are indicative and rely on standard assumptions from approximation theory. They are intended to explain qualitative differences between surrogate classes, rather than provide optimal approximation rates.

\end{document}